\documentclass{article}

\usepackage{microtype}
\usepackage{graphicx}
\usepackage{subfigure}
\usepackage{booktabs} 
\usepackage{authblk}

\usepackage{hyperref}

\usepackage{amsmath}
\usepackage{amssymb}
\usepackage{mathtools}
\usepackage{amsthm}

\usepackage{multicol}

\usepackage[capitalize,noabbrev]{cleveref}

\theoremstyle{plain}

\theoremstyle{definition}

\theoremstyle{remark}

\usepackage[textsize=tiny]{todonotes}

\DeclareUnicodeCharacter{2212}{-}

\usepackage[utf8]{inputenc} 
\usepackage[T1]{fontenc}    
\usepackage{hyperref}       
\usepackage{url}            
\usepackage{booktabs}       
\usepackage{amsfonts}       
\usepackage{nicefrac}       
\usepackage{microtype}      
\usepackage{xcolor}         
\usepackage{graphicx}

\usepackage{amsmath}
\usepackage{amsthm}
\usepackage{amssymb}
\usepackage{nccmath}
\usepackage{bbm}
\usepackage{bm}
\usepackage{mathtools}

\usepackage{comment}
\usepackage{caption}
\usepackage{subcaption}

\usepackage{natbib}

\usepackage{algorithmic}
\usepackage{algorithm}

\newtheorem{thm}{Theorem}

\newtheorem{lem}[thm]{Lemma}

\newtheorem{prop}[thm]{Proposition}

\newtheorem{assume}[thm]{Assumption}

\theoremstyle{definition}

\newcommand{\RR}{ \mathbb{R} }
\newcommand{\NN}{\mathbb{N}}
\newcommand{\CC}{\mathbb{C}}

\newcommand{\spaceo}{\hspace{2 mm}}
\newcommand{\half}{\frac{1}{2}}

\newcommand{\Abs}[1]{\left| #1 \right|}
\newcommand{\Set}[1]{\left\{ #1 \right\}}
\newcommand{\Brack}[1]{\left( #1 \right)}

\newcommand{\SqBrack}[1]{\left[ #1 \right]}
\newcommand{\inner}[2]{\left< #1 , #2 \right>}

\newcommand{\Exp}[1]{ \mathbb{E} #1}

\newcommand{\norm}[1]{\left\|#1\right\|}

\newcommand{\Ind}[1]{ \mathbbm{1}_{\Set{#1}} }

\newcommand{\eps}{\varepsilon}

\DeclareMathOperator*{\argmin}{\arg\!\min}

\newlength{\dhatheight}

\newcommand{\mcF}{\mathcal{F}}
\newcommand{\mc}[1]{\mathcal{#1}}

\newcommand{\setsep}{ \spaceo \vert \spaceo}

\newcommand{\grad}{\nabla}

\newif\ifimagesshow
\imagesshowfalse

\newcommand{\mcX}{\mc{X}}

\newcommand{\mcH}{\mc{H}}

\newcommand{\mcC}{\mc{C}}

\title{Sobolev Space Regularised Pre Density Models}

\author[1]{Mark Kozdoba}
\author[1]{Binyamin Perets}
\author[1]{Shie Mannor}

\affil[1]{The Faculty of Electrical and Computer Engineering, Technion- Israel Institute of Technology}

\date{}

\begin{document}

\maketitle

\begin{abstract}
We propose a new approach to non-parametric density estimation that is based on regularizing a Sobolev norm of the density. This method is statistically consistent, and makes the inductive bias of the model clear and interpretable. While there is no closed analytic form for the associated kernel, we show that one can approximate it using sampling. The optimization problem needed to determine the density is non-convex, and standard gradient methods do not perform well. However, we show that with an appropriate initialization and using natural gradients, one can obtain well performing solutions. Finally, while the approach provides pre-densities (i.e. not necessarily integrating to 1), which prevents the use of log-likelihood for cross validation, we show that one can instead adapt Fisher divergence based score matching methods for this task. We evaluate the resulting method on the comprehensive recent anomaly detection benchmark suite, ADBench, and find that it ranks second best, among more than 15 algorithms.
\end{abstract}

\section{Introduction}
\label{sec:intro}
Density estimation is one of the central problems in statistical learning. In recent years, there has been a tremendous amount of work in the development of parametric neural network based density estimation methods, such as Normalizing Flows \cite{papamakarios2021normalizing}, Neural ODEs \cite{chen2018neural}, and Score Based methods, \cite{song2021scorebased}. However, the situation appears to be different for non parametric density estimation methods, \cite{wasserman2006all}, \cite{hardle2004nonparametric}. 
While there is recent work for low dimensional (one or two dimensional) data,  see for instance \cite{takada2008asymptotic}, \cite{uppal2019nonparametric}, \cite{cui2020nonparametric}, \cite{marteau2020non}, \cite{ferraccioli2021nonparametric}  (see also the survey \cite{kirkby2023spline}), there still are very few non-parametric methods applicable in higher dimensions. 

Compared to parametric models, non parametric methods are often conceptually simpler, and the model bias (e.g., prior knowledge, type of smoothness) is explicit. This may allow better interpretability and better regularization control in small and medium sized data regimes. 

Ideally, a function $g$ produced by a density estimator should satisfy the following properties: \emph{(i)} It is non-negative, \emph{(ii)} it is integrable, and integrates to 1, \emph{(iii)}  it models well multidimensional data, and  \emph{(iv)} it is computationally feasible. 

As mentioned above, to the best of our knowledge currently there are no 
non-parametric estimators that satisfy \emph{(iii)} and \emph{(iv)}. Moreover, note that even the requirements \emph{(i)} and \emph{(ii)} are non-trivial in the sense that they are not natural in the context of typical non-parametric constructions. Indeed, some of the above mentioned estimators simply forego either or both of these conditions even in 1d situations (see also a detailed discussion in \cite{marteau2020non} on this topic).

In this paper we develop the first non-parametric estimator that satisfies all the properties above, with the exception of integrating to 1 (but still integrable). Since many tasks derived from density estimation do not require the precise normalisation constant knowledge, this yields a practically useful and theoretically well founded method.

Let $\mc{S} = \Set{x_i}_{i=1}^N \subset \RR^d$ be a set of data points sampled i.i.d from some unknown distribution. We study a density estimator of the following form: 
\begin{equation}
\label{eq:estimate_def1}
f^* :=  \argmin_{f \in \mcH^{a}}  -\frac{1}{N}\sum_{i=1}^N \log f^2(x_i) + 
\norm{f}^2_{\mcH^{a}}.
\end{equation} 
Here $\mcH^{a}$ is a Sobolev type Reproducing Kernel Hilbert Space (RKHS) of functions, having a norm of the form
\begin{equation}
\label{eq:norm_def1}
\norm{f}_{\mcH^{a}}^2  = \int_{\RR^d} f^2(x) dx + a     
\int_{\RR^d} \Abs{(Df)}^2(x) dx,
\end{equation}
where $D$ represents a combination of derivatives of a certain order. 
The density estimate is given by the function $(f^*)^2$. Note that 
$(f^*)^2$ is clearly non-negative, and $\norm{f}_{\mcH^{a}} < \infty$ implies $\int_{\RR^d} (f^*)^2(x) dx < \infty$. Thus, $(f^*)^2$ is \emph{integrable} over $\RR^d$, although not necessarily integrates to 1. In this paper we refer to such functions as \emph{pre-densities}. To convert such functions to densities, it sufficient to multiply them by the appropriate constant. Note also that \eqref{eq:estimate_def1} is essentially a regularized maximum likelihood estimate, where in addition to bounding the total mass of $(f^*)^2$, we also bound the norm of some of the derivatives of $f^*$. The fact that $\mcH^{a}$ is an RKHS allows us to compute $f^*$ via the standard Representer Theorem. Observe that it would not be possible to control only the  $L_2$ norm of $f^*$ and maintain computabilty, since $L_2$ is not an RKHS. However, adding the derivatives with any coefficient $a>0$ makes the space into an RKHS which allows to control smoothness, hence, 
we call the objective SObolev Space REgularised Pre density estimator (SOSREP). 


The objective \eqref{eq:estimate_def1} has been introduced in \cite{goodd1971nonparametric} and further studied in \cite{klonias1984class}, in the context of spline based methods in \emph{one} dimension.  In this paper we generalise this approach to arbitrary dimensions, which requires resolving several theoretical and computational issues.

On the side of the theory, generalising the existing 1d results, we prove the asymptotic consistency of the SOSREP estimator for the SDO kernels in any fixed dimension $d$, under mild assumptions on the ground truth density generating the $x_i$'s.  In addition, we present a family of examples in which the SOSREP and the standard Kernel Density Estimator (KDE, \cite{wasserman2006all}, with the same kernel) provably  arbitrarily differ. Thus, SOSREP is a genuinely new estimator, non equivalent to KDE, and having some different properties. We also show that examples as above occur naturally in real datasets.  Due to space constraints, the discussion of these examples is deferred to the Supplementary Material, Section \ref{sec:two_regions}.

On the computation side, we resolve several issues: First, for $d>1$ the kernel corresponding to $\mcH^{a}$, which we call the SDO kernel (Single Derivative Order; see Section \ref{sec:SDO_kernel_approximation}), no longer has an analytical expression. We show however that it can, nevertheless, be approximated by an appropriate sampling procedure.

Next, the problem \eqref{eq:estimate_def1} is not convex in $f$, and 
we find that standard gradient descent optimization applied naively produces poor results. We show that this may be resolved by an appropriate initialization, and further improved by using a certain \emph{natural gradient} optimisation rather than the standard one. Specifically, we will show that in our setup the natural gradient preserves some positivity properties, which are crucial for finding good solutions. Curiously, we believe this is one of the very few instances where it is clear \emph{why} natural gradients perform better.

Further, note that as a result of consistency, it can also be shown that the estimator $(f^*)^2$ becomes normalized, at least asymptotically with $N$ (see Section \ref{sec:consistency_statement}). However,  for finite $N$, computing or even estimating the normalization constant is not straightforward, and is outside the scope of this paper. Instead, we will focus on the Anomaly Detection (AD) applications, which do not require a normalization. 
A discussion of possible applications of SOSREP to another key task, that of generative modelling, is given in Section \ref{sec:conclusion}.

Next, although the normalisation may not be required for the main task itself, lack of normalisation may still introduces a particular nuisance in the context of hyperparameter tuning, as it prevents the utilization of the maximum likelihood measure to establish the optimal ``bandwidth" parameter $a$. To resolve this, instead of the likelihood we consider the Fisher Divergence (FD), which uses log-likelihood \emph{gradients} for divergence measurement, thereby eliminating the need for normalization. More specifically, we adapt the concept of score-matching \citep{hyvarinen2005estimation,song2020sliced,song2019generative}, a technique that has recently garnered renewed interest, to our setting.

Finally, building on the above steps, we show that SOSREP achieves the remarkable performance of scoring \emph{second best} on a recent comprehensive anomaly detection benchmark, \cite{han2022adbench} for tabular data, which includes 47 datasets and 15 specialized AD methods.

Interestingly, it is worth noting that anomaly detection has been regarded in the literature as a particularly difficult task, especially for deep density estimators  \cite{nalisnick2019deep, choi2019waic}, and the difficulty persists even in the presence of substantial amounts data. We believe that this makes the good performance of a much more interpretable method, such as SOSREP, even more remarkable. 

The rest of the paper is organized as follows: Section \ref{sec:literature} reviews related literature. Section \ref{sec:INER_density_estimator} introduces the RSR estimator and treats associated optimization questions. The SDO kernel and the associated sampling approximation are discussed in Section \ref{sec:SDO_kernel_approximation}. Consistency results are stated in Section \ref{sec:consistency_statement}. Section \ref{sec:Exp} contains the experimental results, and Section \ref{sec:conclusion} concludes the paper.

\section{Literature and Related Work}
\label{sec:literature}
As discussed in Section \ref{sec:intro}, a scheme that is equivalent to \eqref{eq:estimate_def1} was studied in \cite{goodd1971nonparametric} and \cite{klonias1984class}; 
see also \cite{eggermont2001maximum}.  However, these works concentrated solely on 1d case, and used spline methods to solve \eqref{eq:main_optimization_problem} in the special case that amounts to the use of one particular kernel. More recently, an estimator that is closely related to \eqref{eq:estimate_def1} was considered in 
\cite{ferraccioli2021nonparametric}, 
but was restricted to the  2 dimensional case, both theoretically and practically. In particular, to obtain solutions, their approach invlolves   discretizing (triangulating) the domain of interest in $\RR^2$. This would not be feasible in any higher dimension.  In contrast to these approaches, our more general RKHS formulation in Section \ref{sec:basic_framework} allows the use of a variety of kernels, and our optimisation algorithm is suitable, and was evaluated on,  high dimensional data.

The work that is perhaps the most closely related to ours is \cite{marteau2020non}, where the authors consider quadratic functions and regularisation with a version of an RKHS norm. The differences with respect to our work are both in scope and in details of the construction that enable computation. First, from a computational viewpoint,  the objective in \cite{marteau2020non} requires optimisation over a certain family of non-negative matrices. While considering matrices has some advantages in terms of convexity, for a dataset of size $N$ this requires $N^2$ parameters to encode such a matrix (their Theorem 1), and $O(N^3)$ computational cost per step (\cite{marteau2020non} top of page 6). Further, their optimisation is a constrained one, which makes the problem significantly more difficult. It thus would be unpractical to apply their methods to something like the ADBench benchmark, and indeed, they have only experimented with a 1-dimensional Gaussian example with $N=50$. In contrast, similarly to the standard kernel methods, our objective only requires $N$ parameters, the optimisation step is $O(N^2)$, and the optimisation can be done with standard optimisers, such as GD or Adam. Correspondingly, our evaluation is done on a state-of-the-art outlier detection dataset. 
Second, while \cite{marteau2020non} is concerned with non negative functions generally, we focus on density estimation much more closely. In particular, we prove consistency of our estimator, which was not shown in \cite{marteau2020non} for the estimator introduced there. In addition, we  prove, both theoretically and through empirical evaluation, that our method is different from KDE, which is the most common non parametric density estimator.  

Another common group of non parametric density estimators are the   \emph{projection methods}, \cite{wainwright2019high}. These methods have mostly been studied in one dimensional setting, see the survey \cite{kirkby2023spline}. It is worth noting that with the exception of \cite{uppal2019nonparametric}, the estimators produced by these methods are not densities, in the sense that they do not integrate to 1, but more importantly, may take negative values. In the context of minmax bounds, projection methods in high dimensions were recently analyzed in \cite{singh2018nonparametric}, extending a classical work \cite{kerkyacharian1993density}. However, to the best of our knowledge, such methods have never been practically applied in high dimensions.


Consistency of the SOSREP in one dimension was shown in \cite{klonias1984class}, for kernels that coincide with SDO in one dimension. Our results provide $L_2$ consistency of the SOSREP in any fixed dimension $d$.  While our approach generally follows the same lines as that of \cite{klonias1984class},  some of the estimates are done differently, since the corresponding arguments in \cite{klonias1984class} were intrinsically one dimensional.

\section{The SOSREP Desnity Estimator}
\label{sec:INER_density_estimator}
In this Section we describe the general SOSREP Density Estimation Framework, formulated in an abstract Reproducing Kernel Hilbert Space. 
We then introduce and resolve issues related to the optimisation of the SOSREP objective, and discuss connections between convexity, positivity, and natural gradients. 

\subsection{The Basic Framework}
\label{sec:basic_framework}
Let $\mcX$ be a set and let $\mcH$ be a Reproducing Kernel Hilbert Space (RKHS) of functions on $\mcX$, with kernel $k: \mcX \times \mcX \rightarrow \RR$. 
In particular, $\mcH$ is equipped with an inner product 
$\inner{\cdot}{\cdot}_{\mcH}$ and for 
every $x\in \mcX$, the function 
$k(x,\cdot) = k_x(\cdot) :\mcX \rightarrow \RR$ is in 
$\mcH$ and satisfies the reproducing 
property, $\inner{k_x}{f}_{\mcH} = f(x)$
for all $f \in \mcH$. The norm on $\mcH$ is denoted by  $\norm{f}_{\mcH}^2 = \inner{f}{f}_{\mcH}$, and the subscript $\mcH$ may be dropped when it is clear from context.  We refer to \cite{scholkopf2002learning} for a general introduction to RKHS theory.

Given a set of points $S = \Set{x_1, \ldots, x_N} \subset \mcX$, we define the SOSREP estimator as the solution to the following optimization problem:
\begin{equation}
\label{eq:main_optimization_problem}
f^* = \argmin_{f\in \mcH} -\frac{1}{N} \sum_i \log f^2(x_i)  + \norm{f}_{\mcH}^2.    
\end{equation}

As discussed in Section \ref{sec:intro}, for appropriate spaces $\mcH$, the function $(f^*)^2$ corresponds to a pre-density (that is, $\int_{\RR^d} (f^*)^2(x) dx < \infty$, but not necessarily $\int_{\RR^d} (f^*)^2(x) dx = 1$). 
We now discuss a few basic properties of the solution to \eqref{eq:main_optimization_problem}. First, by the Representer Theorem for RKHS, the minimizer of \eqref{eq:main_optimization_problem} has the form 
\begin{equation}
\label{eq:f_kernel_definition}
    f(x) = f_{\alpha}(x) = \sum_{i=1}^N \alpha_i k_{x_i}(x) 
\end{equation}
for some $\alpha = (\alpha_1, \ldots, \alpha_N) \in \RR^N$. 
Thus one can solve \eqref{eq:main_optimization_problem} by optimizing over a finite dimensional vector $\alpha$. Next, it is worth noting that standard RKHS problems, such as regression, typically use the term $\lambda \norm{h}^2_{\mcH}$, where $\lambda>0$ controls the regularization strength. However, due to the special structure of \eqref{eq:main_optimization_problem}, any solution with $\lambda \neq 1$ is a rescaling by a constant of a $\lambda = 1$ solution. Thus considering only $\lambda = 1$ in \eqref{eq:main_optimization_problem} is sufficient. In addition, we note that any solution of \eqref{eq:main_optimization_problem} satisfies $\norm{f}_{\mcH}^2 = 1$. See Lemma \ref{lem:minimization_props} in Supplementary Material for full details on these two points.  

\subsection{Convexity, Positivity, and Natural Gradients}
\label{sec:conv_pos_nat}
Next, observe that the objective 
\begin{equation}
\label{eq:L_f_def}
\begin{split}
    L(f) = -\frac{1}{N} & \sum_i \log f^2(x_i)  + \norm{f}_{\mcH}^2  \\
= -\frac{1}{N} & \sum_i \log \inner{f}{k_{x_i}}_{\mcH}^2  + \norm{f}_{\mcH}^2  
\end{split}
\end{equation} 
is not convex in $f$. This is due to the fact that the scalar function 
$a \mapsto - \log a^2$ from $\RR$ to $\RR$ is not convex and is undefined at $0$.  However, the restriction of $a \mapsto - \log a^2$ 
to $a \in (0, \infty)$ is convex.  Similarly, the restriction of $L$ to the positive cone of functions $\mcC = \Set{ f \setsep f(x) \geq 0 \spaceo \forall x\in \mcX}$ is convex.  Empirically, we have found that the lack of convexity results in poor solutions found by gradient descent. Intuitively, this is caused by $f$ changing sign, which implies that $f$ should pass through zero at some points. If these points happen to be near the test set, this results in low likelihoods.  At the same time, there seems to be no computationally affordable way to constrain the optimization to the positive cone $\mcC$. Indeed,  standard methods for constrained problems, such as the projected gradient descent, \cite{nesterov2003introductory}, would be costly due to the need to compute the projections.

We resolve this issue in two steps: First, we use a non-negative $\alpha$ initialization, $\alpha_i \geq 0$. Note that for $f$ given by \eqref{eq:f_kernel_definition}, if the kernel is non-negative, then $f$ is non-negative. Although some kernels are non-negative, the SDO kernel, and especially its finite sample approximation (Section \ref{sec:sampling_kernel_v1}) may have negative values. At the same time, there are few such values, and empirically such initialization tremendously improves the performance of the gradient descent.  Second, we use the \emph{natural gradient}, defined in the next section. One can show that for non-negative kernels, $\mcC$ is in fact invariant under natural gradient steps (supplementary material Section \ref{sec:invariance_of_cone_under_nat}). That is, if for $f_{\alpha}$ given by \eqref{eq:f_kernel_definition} we have $f_{\alpha} \in \mcC$, and $\alpha'$ is obtained by the natural gradient step \eqref{eq:gradient_steps_in_alpha_coords_natural}, then $f_{\alpha'} \in \mcC$. This does not seem to be true for the regular gradient. As a consequence, at least for fully non-negative kernels, by using natural gradient we can perform optimisation in $\mcC$, \emph{without the need for computationally costly constraints}. As for the SDO kerenl, although it is not completely non-negative, we observe empirically that the use of natural gradient results in a more stable algorithm and in performance improvement compared to the standard gradient descent. A comparison of standard and natural gradients w.r.t stability to negative values is given in Section \ref{sec:positive_cone_nat_vs_alpha_test}.

\subsection{Explicit Form of Gradients}
\label{sec:gradients}
In this Section we explicitly compute the regular and the natural gradients, both in the function space and in terms of $\alpha$.

We are interested in the minimization of $L(f)$, defined by \eqref{eq:L_f_def}. Using the representation \eqref{eq:f_kernel_definition}
for $x \in \mcX$, we can equivalently consider minimization in $\alpha \in \RR^N$. Let $K = \Set{k(x_i,x_j)}_{i,j\leq N} \in \RR^{N \times N}$ denote the empirical kernel matrix. Then standard computations show that $\norm{f_{\alpha}}^2_{\mcH} = \inner{K\alpha}{\alpha}_{\RR^N}$ and we have $(f_{\alpha}(x_1), \ldots, f_{\alpha}(x_N)) = K\alpha$ (as column vectors). Thus one can consider $L(f_{\alpha}) = L(\alpha)$ 
as a functional $L : \RR^N \rightarrow \RR$ and explicitly compute the gradient w.r.t $\alpha$. This gradient is given in \eqref{eq:grad_alpha}. 

As detailed in \ref{sec:basic_framework}, it is also useful to consider the Natural Gradient -- the gradient of $L(f)$ as a function of $f$, directly in the space $\mcH$. Briefly, a directional Fr\'echet derivative, \cite{munkres2018analysis}, of $L$ at point $f \in \mcH$ in direction $h \in \mcH$ is defined as the limit 
$D_hL(f) = \lim_{\eps \rightarrow 0} \eps^{-1}\cdot \Brack{L(f+\eps h) -L(f)}$. As a function of $h$, $D_hL(f)$ can be shown to be a bounded and linear functional, and thus by the Riesz Representation Theorem, there is a vector, which we denote $\grad_f L$, such that 
    $D_h L(f) = \inner{\grad_f L}{h}$  for all $h \in \mcH$.
We call $\grad_f L$ the Natural Gradient of $L$, since its uses the native space $\mcH$. Intuitively, this definition parallels the regular gradient definition, but uses the $\mcH$ inner product to define the vector $\grad_f L$, instead of the standard, ``parametrization dependent'' inner product in $\RR^N$, that is used to define $\grad_{\alpha} L$.  For the purposes of this paper, it is sufficient to note that similarly to the regular gradient, the natural gradient satisfies the chain rule, and we have $\grad_f \norm{f}^2_{\mcH} = 2f$ and $\grad_f \inner{g}{f}_{\mcH} = g$ for all $g\in \mcH$. The explicit gradient expressions are given below:
\begin{lem}[Gradients]
\label{lem:gradients}
The standard and the natural gradients of $L(f)$ are given by 
\begin{equation}
\label{eq:grad_alpha}
\begin{split}
    &\grad_{\alpha} L = 2 \SqBrack{ K \alpha - \frac{1}{N} K \Brack{K\alpha}^{-1}} \in \RR^N \text{ and } \\ 
    & \grad_{f} L = 2 \SqBrack{ 
       f - \frac{1}{N} \sum_{i=1}^N f^{-1}(x_i) k_{x_i}
    } \in \mcH
\end{split}
\end{equation}
where for a vector $v\in \RR^d$, $v^{-1}$ means coordinatewise inversion. 
\end{lem}
If one chooses the functions $k_{x_i}$ as a basis for the space $H_S = span \Set{k_{x_i}}_{i\leq N} \subset \mcH$, then $\alpha$ in \eqref{eq:f_kernel_definition} may be regarded as coefficients in this basis. For $f = f_{\alpha} \in H_S$ one can then write 
in this basis $\grad_{f} L = 2 \SqBrack{ \alpha - \frac{1}{N} (K\alpha)^{-1}} \in \RR^N$. 
Therefore in the $\alpha$-basis we have the following standard and natural gradient iterations, respectively: 
\begin{equation}
\label{eq:gradient_steps_in_alpha_coords_standard}
 \alpha' \leftarrow \alpha - 2 \lambda     
\SqBrack{ K \alpha - \frac{1}{N} K \Brack{K\alpha}^{-1}} \text{ and }  
\end{equation}
\begin{equation}
\alpha' \leftarrow \alpha - 2 \lambda     
\SqBrack{\alpha - \frac{1}{N} \Brack{K\alpha}^{-1}},    
\label{eq:gradient_steps_in_alpha_coords_natural}
\end{equation}
where $\lambda$ is the learning rate.

\section{Single Derivative Order Kernel Approximation}
\label{sec:SDO_kernel_approximation}
In this Section we introduce the Single Derivative Order kernel, which corresponds to norms of the form \eqref{eq:norm_def1} discussed in Section \ref{sec:intro}. In Section \ref{sec:kernel_integral_form} we introduce the relevant Sobolev functional spaces and derive the Fourier transform of the norm. In Section \ref{sec:sampling_kernel_v1} we describe a sampling procedure that can be used to approximate the SDO.

\subsection{The Kernel in Integral Form}
\label{sec:kernel_integral_form}
For a function $f:\RR^d \rightarrow \CC$ and a tuple $\kappa \in \Brack{\NN \cup \Set{0}}^d$,
let  
$D^{\kappa} = \frac{\partial f}{\partial x_1^{\kappa_1} \ldots \partial x_d^{\kappa_d}} $ 
denote the $\kappa$ indexed derivative. By convention, for $\kappa = (0,0,\ldots, 0)$ we set $D^{\kappa} f = f$.
Set also $\kappa! = \prod_{j=1}^d \kappa_j!$ and $\Abs{\kappa}_1 = \sum_{j=1}^d \kappa_j$. Set $\norm{f}_{L_2}^2 = \int \Abs{f(x)}^2 dx $. Then, for $m \in \NN$ and $a>0$ denote 
\begin{equation}
\label{eq:norm_definition_formal}
\norm{f}_a^2 = \norm{f}_{L_2}^2 + a  \sum_{|\kappa|_1 = m} \frac{m!}{\kappa!} \norm{\Brack{D^{\kappa} f}}_{L_2}^2.
\end{equation}
The norm $\norm{f}_a^2$ induces a topology that is equivalent to that of a standard $L_2$ Sobolev space of order $m$. We refer to \cite{adams2003sobolev}, \cite{saitoh2016theory} for background on Sobolev spaces. However, here we are interested in properties of the norm that are  finer than the above equivalence. For instance, note that for all $a\neq 0$ the norms $\norm{f}_a$ are mutually equivalent, but nevertheless, a specific value of $a$ is crucial in applications, for regularization purposes. 

Let $\mcH^a = \Set{f :\RR^d \rightarrow \CC \setsep \norm{f}_a^2 < \infty}$ be the space of functions with a finite $\norm{f}_a^2$ norm.
Denote by 
\begin{equation}
\label{eq:norm_definition_formal_inner_product}
    \inner{f}{g}_{\mcH^a} = \inner{f}{g}_{L_2} + 
    a  \sum_{|\kappa|_1 = m} \frac{m!}{\kappa!} \inner{\Brack{D^{\kappa} f}}{\Brack{D^{\kappa} g}}_{L_2}
\end{equation} 
the inner product that induces the norm $\norm{f}_a^2$.
\begin{thm}
\label{thm:kernel_form}
For $m > d/2$ and any $a>0$, the space $\mcH^a$ admits a reproducing kernel $k^a(x,y)$ satisfying $\inner{k^a_x}{f}_{\mcH^a} = f(x)$ for all $f\in \mcH^a$ and $x\in \RR^d$. 
The kernel is given by 
\begin{equation}
\label{eq:kernel_fourier_inv}    
\begin{split}
    k^a(x,y) &= \int_{\RR^d} \frac{e^{2\pi i \inner{y-x}{z}}}{
     1 +   
    a  \cdot (2\pi)^{2m} \norm{z}^{2m} 
} dz =  \\
& \int_{\RR^d} \frac{1}{
     1 +   
    a  \cdot (2\pi)^{2m} \norm{z}^{2m} 
} \cdot  e^{2\pi i \inner{y}{z}} \cdot \overline{e^{2\pi i \inner{x}{z}}} dz.
\end{split}
\end{equation}
\end{thm}
The proof of Theorem \ref{thm:kernel_form} follows the standard approach of deriving kernels in Sobolev spaces, via computation and inversion of the Fourier transform, see \cite{saitoh2016theory}. However, compact expressions such as  \eqref{eq:kernel_fourier_inv} are only possible for some choices of derivative coefficients. Since the particular form \eqref{eq:norm_definition_formal} was not previously considered in the literature (except for $d=1$, see below), we provide the full proof in the Supplementary Material. 

\subsection{Kernel Evaluation via Sampling}
\label{sec:sampling_kernel_v1}
To solve the optimization problem \eqref{eq:main_optimization_problem} in $\mcH^a$, we need to be able to evaluate the kernel $k^a$ at various points. For $d=1$, closed analytic expressions were obtained in cases $m=1,2,3$ in \cite{thomas1996computing}. In particular, for $m=1$, $k^a$ coincides with the Laplacian kernel $k_h(x,y) = e^{-h \Abs{x-y}}$. 
However, for $d>1$, it seems unlikely that there are closed expressions. See \cite{novak2018reproducing} for a discussion of this issue for a similar family of norms.

To resolve this, note that the form \eqref{eq:kernel_fourier_inv} may be interpreted as an average of the terms $e^{2\pi i \inner{y}{z}} \cdot \overline{e^{2\pi i \inner{x}{z}}}$, where 
$z$ is sampled from an unnormalized density 
$w^a(z) = (1 +   a  \cdot (2\pi)^{2m} \norm{z}^{2m})^{-1} $ on $\RR^d$.  This immediately suggest that if we can sample from 
$w^a(z)$, then we can approximate $k^a$ by summing over a finite set of samples $z_j$ instead of computing the full integral.

In fact, a similar scheme was famously previously employed in \cite{rahimi2007random}, in a different context. There, it was observed that by Bochners's Theorem, \cite{rudin2017fourier}, any stationary kernel can be represented as $k(x,y) = \int \nu(z) e^{2\pi i \inner{y}{z}} \cdot \overline{e^{2\pi i \inner{x}{z}}} dz$ for some non-negative measure $\nu$. Thus, if one can sample $z_1,\ldots,z_T$ from $\nu$, one can construct an approximation
\begin{equation}
\label{eq:kernel_sample_approx}
    \hat{k}^a(x,y) = \frac{1}{T} \sum_{t=1}^T 
    \cos\Brack{\inner{z_t}{x} +b_t} \cdot \cos\Brack{\inner{z_t}{y} +b_t},
\end{equation}
where $b_t$ are additional i.i.d samples, sampled uniformly from 
$[0,2\pi]$.  In \cite{rahimi2007random}, this approximation was used as a dimension reduction for \emph{known} analytic kernels, such as the Gaussian, for which the appropriate $\nu$ are known. Note that the samples $z_t,b_t$ can be drawn once, and subsequently used for all $x,y$ (at least in a bounded region, see the uniform approximation result in \cite{rahimi2007random}).  

For the case of interest in this paper, the SDO kernel, Bochner's representation is given by \eqref{eq:kernel_fourier_inv} in Theorem \ref{thm:kernel_form}. Thus, to implement the sampling scheme \eqref{eq:kernel_sample_approx} it remains to describe how one can sample from the density $w^a(z)$ on $\RR^d$. To this end, note that 
$w^a(z)$ is spherically symmetric, and thus can be decomposed as 
$z = r \theta$, where $\theta$ is sampled uniformly from a unit sphere $S^{d-1}$ and the radius $r$ is sampled from a \emph{one dimensional} density $u^a(r) = \frac{r^{d-1}}{1 + a (2\pi r)^{2m}}$ (see the Supplementary Material for full details on this change of variables).  Next, note that sampling $\theta$ is easy. Indeed, let $g_1,\ldots,g_d$ be i.i.d standard Gaussians. Then $\theta \sim (g_1,\ldots,g_d)/\sqrt{\sum_i g_i^2}$.  Thus the problem is reduced to sampling a one dimensional distribution with a single mode, with known (unnormalized) density. This can be efficiently achieved by methods such as Hamiltonian Monte Carlo (HMC). However, we found that in all cases a fine grained enough discretization of the line was sufficient.

\begin{figure*}[t]
    \centering
    \includegraphics[scale =0.35]{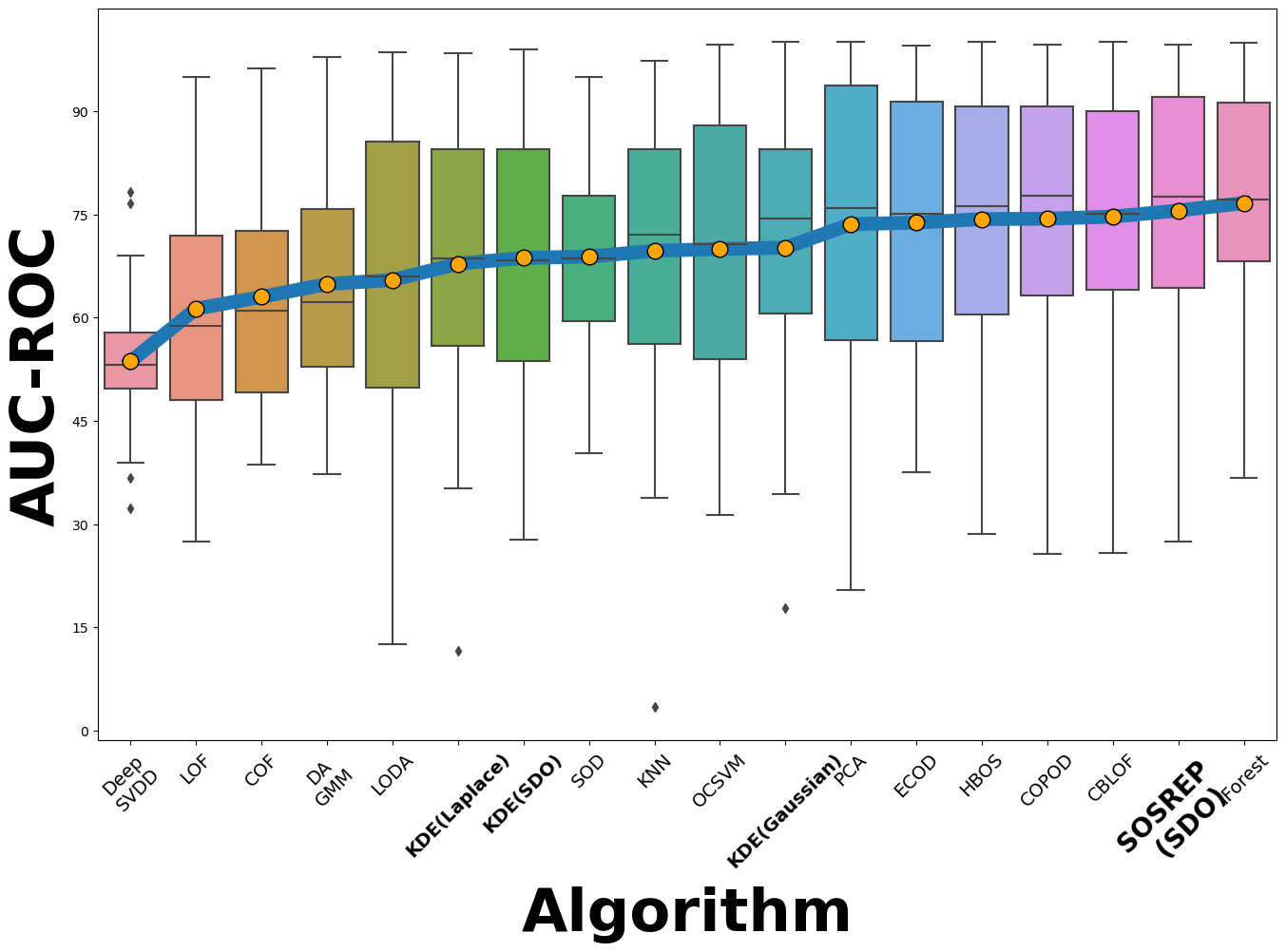}
    \caption{Anomaly Detection Results on ADBench, higher is Better. SOSREP is Second Best Among 18 Algorithms}
    \vspace{-5.00mm} 
    \label{fig:ADbench_aucs}
\end{figure*}

\newcommand{\Mtrn}{Mat\'{e}rn }
\subsection{Relations With Other Kernels} 
We now discuss the relation between SDO and the Laplacian, Gaussian and \Mtrn kernels.

The SDO kernel with parameter $m$ is defined by the norm
\eqref{eq:norm_definition_formal}, which involves only derivatives of order $m$ and we use $m = \lceil d/2 \rceil$ in the experiments. 
On the other hand, it is well known that the Gaussian kernel, $k(x,y) = e^{-\frac{\|x-y\|^2}{2\sigma^2}}$, corresponds to a Sobolev space with a norm that is analogous to \eqref{eq:norm_definition_formal}, but involves a summation over \emph{all} orders of derivatives (see \cite{williams2006gaussian}, Chapter 6), and thus contains only infinitely smooth functions.  More recently, it has been shown that the Laplacian kernel, $k(x,y) = e^{-\frac{\|x-y\|}{2\sigma}}$, corresponds to a Sobolev space with a norm that involves all the derivatives up to order $m = \lceil d/2 \rceil$, (see \cite{rakhlin2019consistency}, Proposition 7). This norm, however, explicitly includes lower order derivatives ($m<<d/2$) terms. Thus bounding the norm also explicitly bounds the lower derivative magnitudes.  In this sense, the Laplace kernel penalises non-smoothness more heavily than the SDO. In fact, SDO with $m = \lceil d/2 \rceil$ may be viewed as the RKHS space with \emph{minimal} smoothness requirements in $d$ dimensions. Indeed, $m> d/2$ is necessary for the space to be an RKHS (otherwise \eqref{eq:kernel_fourier_inv} diverges, see 
Supplementary Sec. 
\ref{sec:full_kernel_derivation_proof} and \ref{sec:SDO_sampling_approximation_details} for additional details), and no additional terms are present in the norm. 

Finally, it is worth noting that the \Mtrn family of kernels with a smoothness parameter $\nu$ \cite{williams2006gaussian} contains the Laplacian and Gaussian kernels as two extremes, with $\nu = \half$ and 
$\nu \rightarrow \infty$, respectively. As shown in  
\cite{fasshauer2011reproducing}, in fact the corresponding RKHS spaces have increasingly growing smoothness order, from $\lceil d/2 \rceil$ to $\infty$, when $\nu = \half + p$, and $p$ varies in $[0,\infty)$.

In addition to the SDO kernel, we have evaluated the SOSREP with Gaussian and Laplace kernels on the Anomaly Detection task and full results can be found in the Supplementary Material \ref{Supp:diverse_Kernels}. While the Gaussain and Laplacian kernels generally perform competitively and, in most cases, SOSREP estimators are better than KDE estimators, the performance of SDO is  better. The precise reasons for this are not clear at the moment, and we believe that a detailed study of the interaction between the specific kernel and SOSREP is of interest. However, it is out of the scope of the present paper and is left as a future work.

\section{Consistency}
\label{sec:consistency_statement}
In this Section we describe the consistency result for the SOSREP estimators with kernels of the form \eqref{eq:kernel_fourier_inv}.

Recall from the  discussion in Section \ref{sec:conv_pos_nat} that the objective $L$,  given by \eqref{eq:L_f_def}, is convex when restricted to the positive cone. Here we consider the following positive cone:
\begin{equation}
\label{eq:cone_c_prime_definition_main}
    \mcC' = \mcC'_a = \Set{ f \in span\Set{k^a_{x_i}}_{i=1}^N \setsep f(x_i) >0 \spaceo \forall i\leq N}.
\end{equation}
Compared to Section \ref{sec:conv_pos_nat}, belonging to $\mcC'$ requires only positivity on the data points $x_i$ rather than on all $x\in \RR^d$, i.e. belonging to $\mcC'$ is a weaker requirement. The convexity of $L$ still holds, however, due to similar considerations. In addition, we restrict the cone to the span of $x_i$ since the optimal solutions belong to that span, and our algorithms operate inside the span too. 

Note also that $f\in \mcH^1$ if and only if $f\in \mcH^a$ for every $a>0$. With these remarks in mind, we can state the consistency result: 

\begin{thm}
\label{thm:consistency}
Let $x_1,\ldots, x_N$ be i.i.d samples from a compactly supported density $v^2$ on $\RR^d$, such that $v\in \mcH^1$.  Set $a = a(N) = 1/N$, and let 
$u_N = u(x_1, \ldots, x_N; a(N))$ be the minimizer of the objective  
\eqref{eq:L_f_def} in the cone \eqref{eq:cone_c_prime_definition_main}. Then $\norm{u_N - v}_{L_2}$ converges to $0$ in probability. 
\end{thm}
In words, when $N$ grows, and the regularisation size $a(N)$ decays as $1/N$, the SOSREP estimators $u_N$ converge to $v$ in $L_2$. 

Note that since $\norm{v}_{L_2} = 1$ (as its a density), and since $\norm{u_N - v}_{L_2} \rightarrow 0$, it follows by the triangle inequality that $\norm{u_N}_{L_2} \rightarrow 1$. That is, the estimator $u_N$ becomes approximately normalized as $N$ grows.

An overview and full details of the proof are provided in Supplementary Material Sections \ref{sec:overview_of_proof} and  \ref{sec:proof_details}.

\section{Experiments}\label{sec:Exp}
In this section we present the evaluation of SOSREP on the ADBench anomaly detection benchmark, and empirically test the advantage of natural gradient descent for maintaining a non-negative $f$. 

\subsection{Anomaly Detection Results for ADbench}
\label{sec:experiments_AD}
ADbench, \cite{han2022adbench}, is a recent Anomaly Detection (AD) benchmark in which 15 state of the art AD algorithms are evaluated on 47 tabular datasets. In addition to these baselines, we evaluate also the standard KDE estimators, using both Gaussian and Laplace kernels. The performance of all these algorithms is compared to that of SOSREP with the SDO kernel.

We focus on the classical unsupervised setup, in which at train time there is no information regarding which samples are anomalies. On test, anomaly labels are revealed only for final algorithm evaluation. 
For all density-based approaches (KDE,SOSREP), we employ the negative of the density as the \emph{anomaly score}. That is, low density values would be considered anomalies. In accordance with the ADbench paper, we evaluate each dataset by calculating the average AUC-ROC of that score across four runs, each with a different random seed. The seed  determines the train-test split of the ADbench data, as well as the randomness of the algorithms.

In Figure \ref{fig:ADbench_aucs}, for each algorithm, we present a box plot where the average AUC-ROC across all datasets is indicated by an orange dot. Additionally, the boundaries of the box represent the 25th and 75th quantiles, while the line inside the box denotes the median. The algorithms are sorted by their average AUC-ROC.  
An additional evaluation, which compares the \emph{relative ranks} of the algorithms instead of mean AUC-ROC, may be found in the Supplementary Material, Section \ref{sec:supp_aucroc}. The results there are very similar to those in Fig. \ref{fig:ADbench_aucs}. 
As discussed in \cite{han2022adbench}, evaluation by rank, in addition to that by AUC, may help mitigate  potential evaluation biases arising from variations in dataset difficulty.  

For both AUC-ROC and rank evaluations, \textbf{SOSREP emerges as the 2nd best AD method overall}.
Notably, this is achieved with the generic version of our method, without any pre or post-processing specifically dedicated to AD. In contrast, many other methods are specifically tailored for AD and do include extensive pre and post-processing.

In addition to performing well on the standard ADBench benchmark, and perhaps even more impressively, SOSREP excels also on the more demanding setup of \emph{duplicate anomalies}, which was also extensively discussed in \citep{han2022adbench}. Here, \textbf{SOSREP rises to the forefront as the top AD method} (with an average AUC-ROC of 71.6 for X5 duplicates - a lead of 4$\%$ over the closest contender). This scenario, which is analogous to numerous practical situations such as equipment failures, is a focal point for ADbench's assessment of unsupervised AD methods due to its inherent difficulty, leading to substantial drops in performance for former leaders like Isolation Forest. Detailed explanations are available in the Supplementary Material Section \ref{sec:supp_duplicate_anomalies}. 

\begin{figure}
    \centering
    \includegraphics[width=0.85\linewidth]{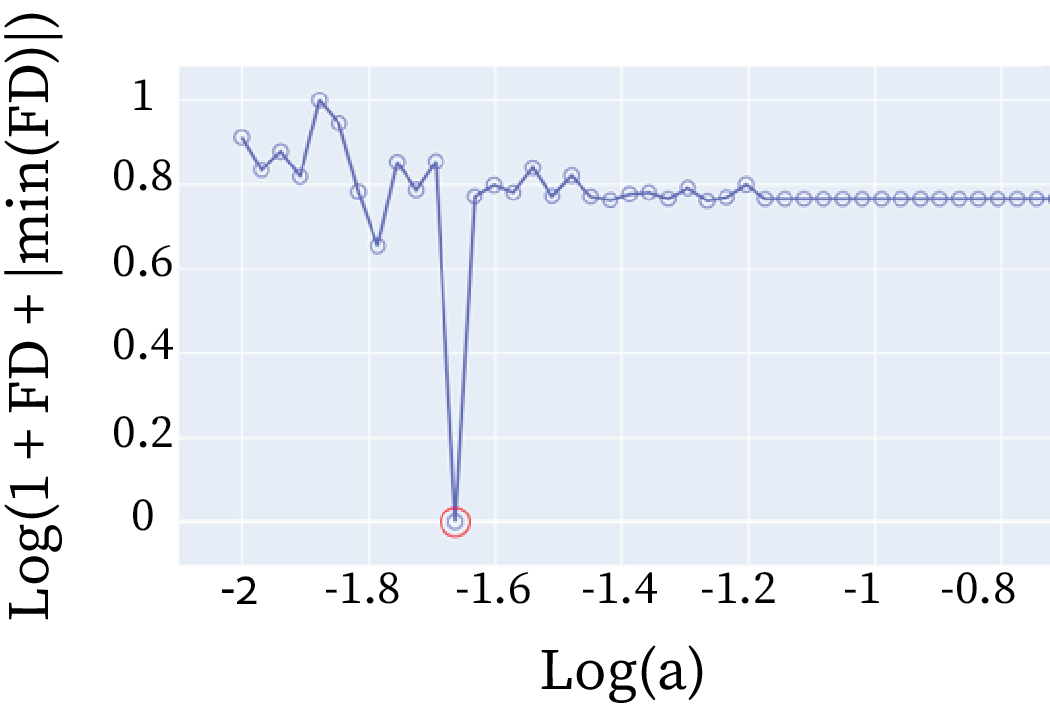}
    \caption{For each HP a, we calculate the Fisher divergence between the density learned on the training set and the density inferred on the test set.}
    \label{fig:enter-label}
\end{figure}

As for computational cost, the entire set of 47 datasets was processed in 186 minutes using a single 3090RTX GPU and one CPU, averaging about 4 minutes per dataset.

\subsection{Fisher-Divergence Based Hyperparameter Tuning}
As discussed in Section \ref{sec:intro}, dealing with unnormalized densities adds some complexity to hyperparameter tuning. 
Indeed, hyperparameter tuning typically involves the comparison of  multiple models, for instance models with different smoothness control (aka ``bandwidth'') parameters $a>0$. In probabilistic models, this would typically involve choosing the model which gives the highest likelihood to the data (note that the anomaly labels are not available at this stage). 
However, due to the lack of normalisation, the likelihoods are not computable. In this paper, instead of likelihoods we compute  the Fisher Diveregences between the model and the data, as these do not require normalisation, and  choose the model with the parameter that yields the lowest divergence (more precisely, we compute the FD up to an additive constant which does not depend on the the model). 
Figure \ref{fig:enter-label} illustrates the FD for SOSREP for different $a$ values for a single dataset. The x-axis shows the parameter values on a logarithmic scale. For visual clarity, we also shifted our version of FD to be strictly positive and applied a logarithmic transformation. One can see that the graph is somewhat noisy, but typically there is  one clear minimum value.

A detailed discussion of the stages involved in the FD computation and the associated minimum selection can be found in Section \ref{SEC:FD} of the supplementary material.


\begin{figure}[t]
    \centering
    \includegraphics[scale=0.4]{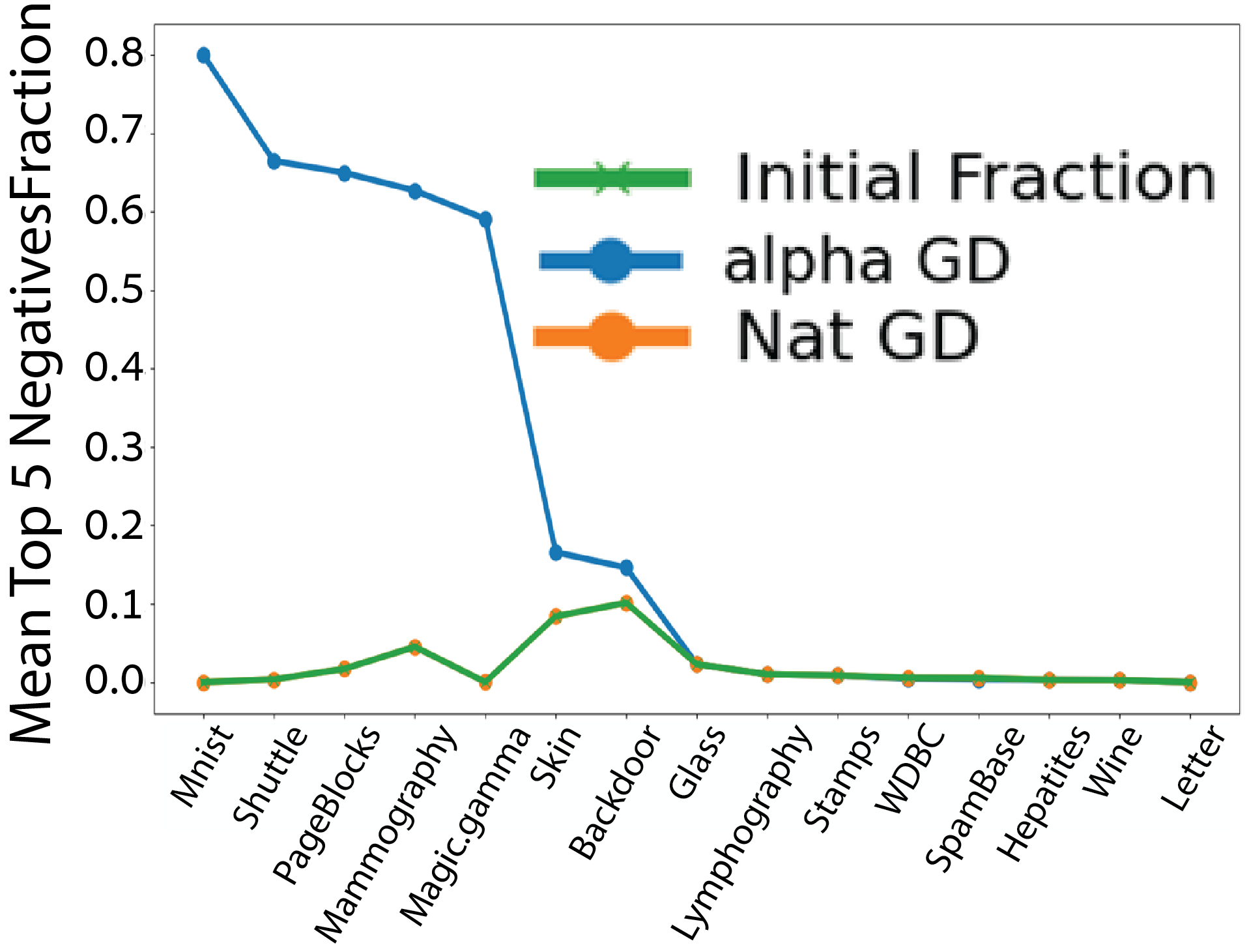}
    \caption{Fraction of negative values for natural versus $\alpha$ gradient-based optimization across datasets. The X-axis represents datasets from ADbench (see Supplementary Material Section \ref{sec:datasets_ng_frc} for details). }
    \label{fig:pos_neg_frac}
\end{figure}

\subsection{Natural-Gradient vs Standard Gradient Comparison}
\label{sec:positive_cone_nat_vs_alpha_test}
In the context of the discussion in Section \ref{sec:basic_framework}, 
we conduct an experiment to demonstrate that the standard gradient descent may significantly amplify the fraction of negative values in a SOSREP solution, while the natural gradient keeps it constant. We have randomly chosen 15 datasets from ADBench, and for each dataset we have used 50 non negative $\alpha$ initializations. Then we have run both algorithms for 1000 iterations. The fraction of negative values 
of $f_{\alpha}$ (on the train set) was measured at initialization, and in the end of each run. In Figure \ref{fig:pos_neg_frac}, for each dataset and for each method, we show an average of the worst 5 fractions among the 50 initializations. Thus, for instance, for the 'shuttle' data, the initial fraction is negligible, and is unchanged by the natural gradient. However, the standard gradient (``alpha GD'' in the figure in blue) for this dataset yields about 70\% negative values in the 5 worst cases (i.e., 10\% of initializations). We conclude that the natural gradient is better at preserving non-negativity than standard gradient descent.

\section{Conclusion}
\label{sec:conclusion}
In this paper we have developed the theory and the optimisation methods for a multi dimensional unnormalised density estimator. This is the first estimator, among those broadly based on variants of functional space normalisation, that is computationally tractable, can be applied in higher dimensions, and performs well on a useful and difficult task. 

In particular, we have shown that the estimator is asymptotically consistent, that it is possible to use sampling to approximate the associated kernel, that it is possible to use a version of natural gradient for optimisation, and Fisher divergence for hyperparameter selection. We have found that the combination of these techniques result in a model that has excellent performance on a state-of-the-art outlier detection benchmark.  

The success of the approach on the ADBench benchmarks indicates that it is capable of modelling complicated densities. It is thus natural to ask next whether the  unnormalised densities introduced in this paper can be used as a basis for generative models. Generative modelling for tabular data of small and intermediate sizes is an active and difficult research area, and SOSREP estimators may provide an alternative approach to the problem. Indeed, note that most Markov Chain Monte Carlo methods (Metropolis Hastings, Langevin and Hamiltonian MC, etc.) work with unnormalised densities by design and can in principle be applied to SOSREP estimated densities out of the box to generate samples. We believe that evaluating such approaches, and designing schemes to alleviate the ``local minima'' issues common to MCMC, is an important and promising direction for future research that is based on the present work. 

\section{Impact Statement}
This paper presents work which has the goal of advancing the field of Machine Learning. There are many potential societal consequences of our work, none of which we feel must be specifically highlighted here.

\bibliographystyle{apalike}
\bibliography{a_density_estimates.bib}

\newpage
\newpage
\onecolumn
\appendix

\section{Overview Of The Supplementary Material}
This Supplementary Material is organized as follows: 

\begin{itemize}
    \item Some basic properties of the SOSREP objective and of the related minimization problem: Section \ref{sec:basic_properties_of_minimizer}
    \item Derivation of the Gradients of the SOSREP objective: Section \ref{sec:gradients_lemma_proof}
    \item Invariance of the non-negative function cone under natural gradient steps: Section \ref{sec:invariance_of_cone_under_nat}
    \item  SOSREP vs KDE comparison: Sections \ref{sec:two_regions} and \ref{sec:kde_vs_iner_comparison_proofs}
    \item Derivation of the integral form of SDO kernel, proof of Theorem 3: Section \ref{sec:full_kernel_derivation_proof}
    \item Additional details on sampling approximation procedure: Section \ref{sec:SDO_sampling_approximation_details}
    \item Details on the Fisher Divergence estimation for SOSREP Hyperparameter tuning : \ref{SEC:FD}
    \item Comparison of the raw AUC-ROC metric on ADBench data: Section \ref{sec:supp_aucroc} 
    \item Discussion of an additional test regime, with duplicated anomalies: Section \ref{sec:supp_duplicate_anomalies}
    \item On-the-fly hyperparameter tuning procedure that was used to save time by finding the first \emph{stable} local minimum: Section \ref{sec:supp:HP tuning FD}
\end{itemize}

The code used for all the experiments in the paper will be publicly released with the final version of the paper.

\section{Basic Minimizer Properties}
\label{sec:basic_properties_of_minimizer}
As discussed in Section \ref{sec:basic_framework}, the minimizer of the SOSREP objective \eqref{eq:main_optimization_problem} always has $\mcH$ norm 1. In addition, there is no added value in multiplying the norm by a regularization scalar, since this only rescales the solution. Below we prove these statements. 
\begin{lem}
\label{lem:minimization_props}
Define 
\begin{equation}
\label{eq:minimization_lemma_objective}
    f = \argmin_{h\in \mcH} -\frac{1}{N} \sum_i \log h^2(x_i) + \norm{h}_{\mcH}^2.
\end{equation}
Then $f$ satisfies $\norm{f}^2 = 1$. 
Moreover, if 
\begin{equation}
    f' = \argmin_{h\in \mcH} -\frac{1}{N} \sum_i \log h^2(x_i)  + \lambda^2 \norm{h}_{\mcH}^2,
\end{equation}
for some $\lambda >0$, 
then $f' = \lambda^{-1} f$. 
\end{lem}

\begin{proof}
For any $h\in \mcH$ and $a>0$, 
\begin{flalign}
\label{eq:a_optimization}
    &\argmin_{a>0} -\frac{1}{N} \sum_i \log (ah)^2(x_i)  + \norm{ah}^2 = \\
    &\argmin_a -\frac{1}{N} \sum_i \log h^2(x_i)  -\log a^2 + a^2 \norm{h}^2.
\end{flalign}
Taking derivative w.r.t $a$ we have 
\begin{equation}
    -\frac{2a}{a^2} + 2a \norm{h}^2 = 0. 
\end{equation}
Thus optimal $a$ for the problem \eqref{eq:main_optimization_problem} must satisfy  
 $\norm{ah}^2 = a^2 \norm{h}^2 = 1$. 
To conclude the proof of the first claim, choose $h$ in \eqref{eq:a_optimization} to be the minimizer in \eqref{eq:minimization_lemma_objective}, $h=f$.
Note that if $\norm{f}_{\mcH} \neq 1$, then we can choose $a = \norm{f}_{\mcH}^{-1} \neq 1$ to further decrease the value of the objective contradicting the fact that $f$ is the minimizer.

For the second claim, denoting $g = \lambda h$, 
\begin{flalign}
    & \argmin_{h \in \mcH} -\frac{1}{N} \sum_i \log h^2(x_i)  + \lambda^2 \norm{h}^2      \\        
     &= \lambda^{-1} \argmin_{g  \in \lambda \mcH = \mcH} -\frac{1}{N} \sum_i \log g^2(x_i)  + \norm{g}^2     + \frac{1}{N} \sum_i \log \lambda^2  \\ 
     &= \lambda^{-1} \argmin_{g  \in \mcH} -\frac{1}{N} \sum_i \log g^2(x_i)  + \norm{g}^2  \\ 
     &= \lambda^{-1} f.
\end{flalign}

\end{proof}

\section{Derivation of the Gradients, Proof Of Lemma \ref{lem:gradients}}
\label{sec:gradients_lemma_proof}
In this section we derive the expressions for standard and the natural gradients of the objective \eqref{eq:L_f_def}, as given in Lemma \ref{lem:gradients}.
\begin{proof}[Proof Of Lemma \ref{lem:gradients}]
We first derive the expression for $\grad_{\alpha} L$ in \eqref{eq:grad_alpha}. Recall that 
$\norm{f}_{\mcH}^2 = \inner{\alpha}{K \alpha}_{\RR^N}$ for $\alpha \in \RR^N$, where $K_{ij} = k(x_i,x_j)$. This follows directly from the form \eqref{eq:f_kernel_definition}, and the fact that $\inner{k_x}{k_y} = k(x,y)$ for all $x,y\in \mcH$, by the reproducing property. For this term we have  $\grad_{\alpha} \inner{\alpha}{K \alpha} = 2 K \alpha$. 
Next, similarly by using \eqref{eq:f_kernel_definition},
$\grad_{\alpha} f(x) = (k(x_1,x), \ldots, k(x_N,x))$ for every $x \in \RR^d$.
Finally, we have 
\begin{flalign}
    \grad_{\alpha} \frac{1}{N} \sum_{i=1}^N \log f^2(x_i) &= 
    \frac{1}{N} \sum_{i=1}^N f^{-2}(x_i) \cdot 2 f(x_i) \cdot \grad_{\alpha} f(x_i) \\ 
    &= 
    2 \frac{1}{N} \sum_{i=1}^N f^{-1}(x_i) \cdot \grad_{\alpha} f(x_i) \\ 
    &= 2 \frac{1}{N} K (f(x_1), \ldots, f(x_N))^{-1} \\ 
    &= 2 \frac{1}{N} K \Brack{K\alpha}^{-1}.
\end{flalign}
This yields \eqref{eq:grad_alpha}. 

For $\grad_f L$, we similarly have $\grad_f \norm{f}_{\mcH}^2 = 2 f$, as discussed in section \ref{sec:gradients}.  
Moreover, 
\begin{flalign}    
    \grad_f \frac{1}{N} \sum_{i=1}^N \log f^2(x_i) &= 
    \grad_f \frac{1}{N} \sum_{i=1}^N \log \inner{f}{x_i}^2_{\mcH}  \\
    &= 
    \frac{1}{N} \sum_{i=1}^N \inner{f}{x_i}^{-2}_{\mcH}  
    \cdot 2\inner{f}{x_i}_{\mcH} \cdot 
    \grad_f \inner{f}{x_i}_{\mcH} \\ 
    &= 2\frac{1}{N} \sum_{i=1}^N \inner{f}{x_i}^{-1}_{\mcH}   
    {x_i}.
\end{flalign}
This completes the proof.
\end{proof}

\section{SDO Kernel Details}
\label{sec:sdo_derivation_proof}

In Section \ref{sec:full_kernel_derivation_proof} we provide a full proof of Theorem \ref{thm:kernel_form}, while Section \ref{sec:SDO_sampling_approximation_details} contains additional details on the sampling approximation of SDO. 
\subsection{SDO Kernel Derivation}
\label{sec:full_kernel_derivation_proof}

We will prove a claim that is slightly more general than Theorem \ref{thm:kernel_form}. For a tuple $\bar{a} \in \RR_{+}^m$, define the norm 
\begin{equation}
\label{eq:a_bar_norm_def}
    \norm{f}_{\bar{a}}^2 = \sum_{l=0}^m  
    a_l \sum_{|\kappa|_1 = l} \frac{l!}{\kappa!} \norm{\Brack{D^{\kappa} f}}_{L_2}^2,
\end{equation}
where $D^{\kappa}$ are the $\kappa$-indexed derivative, as discussed in Section \ref{sec:kernel_integral_form}. The SDO norm is a special case with $a_0 = 1$, $a_m = a$, and $a_l = 0$ for $0<l<m$.
Let $\mcH^{\bar{a}}$ be the subspace of $L_2$ of functions with finite norm, 
\begin{equation}
    \mcH^{\bar{a}} = \Set{ f \in L_2 \setsep \norm{f}_{\bar{a}} < \infty}
\end{equation}
and let the associated inner product be denoted by 
\begin{equation}
    \inner{f}{g}_{\bar{a}} = \sum_{l=0}^m  
    a_l \sum_{|\kappa|_1 = l} \frac{l!}{\kappa!} \inner{\Brack{D^{\kappa} f}}{\Brack{D^{\kappa} g}}_{L_2}.
\end{equation}
Define the Fourier transform 
\begin{equation}
\label{eq:fourier_def}
    \mc{F}f(z) = \int_{\RR^d} f(u) e^{-2\pi i  \inner{z}{u}} du,
\end{equation}
and recall that we have (see for instance \cite{stein1971introduction}, \cite{grafakos2008classical})
\begin{equation}
\label{eq:fourier_derivative}
    \mc{F}\Brack{D^{\kappa} f}(z) = \Brack{\prod_{j=1}^d \Brack{2\pi i z_j}^{\kappa_j} } \mc{F}f (z) \text{ for all } z \in \RR^d. 
\end{equation}

The following connection between the $L_2$ and the derivative derived  norms is well known for the standard Sobolev spaces (\citep{williams2006gaussian,saitoh2016theory,novak2018reproducing}).
However, since \eqref{eq:a_bar_norm_def} somewhat differs from the standard definitions, we provide the argument for completeness. 
\begin{lem}
\label{lem:norm_in_fourier}
Set for $z \in \RR^d$
\begin{equation}
    \label{eq:va_def}
    v_{\bar{a}}(z) = \Brack{ 1 + \sum_{l=1}^m  
    a_l  \cdot (2\pi)^{2l} \norm{z}^{2l} 
          }^{\half}.
\end{equation}
Then for every $f \in \mcH^{\bar{a}}$ we have 
    \begin{equation}
        \norm{f}_{\bar{a}}^2 = \norm{v_{\bar{a}}(z) \cdot \mc{F}[f]}^2_{L_2}.
    \end{equation}
\end{lem}
\begin{proof}
\begin{flalign}
    \norm{f}_{\bar{a}}^2 &= \sum_{l=0}^m  
    a_l \sum_{|\kappa|_1 = l} \frac{l!}{\kappa!} \norm{D^{\kappa} f}_{L_2}^2 \\ 
    &= \sum_{l=0}^m  
    a_l \sum_{|\kappa|_1 = l} \frac{l!}{\kappa!} \norm{\mc{F} \SqBrack{D^{\kappa} f}}_{L_2}^2 \\
    &= \int dz \SqBrack{ \sum_{l=0}^m  
    a_l \sum_{|\kappa|_1 = l} \frac{l!}{\kappa!} \Abs{\mc{F} \SqBrack{D^{\kappa} f}(z) }^2}  \\
    &= \int dz \SqBrack{ \Abs{\mc{F}\SqBrack{f}(z)}^2 + \sum_{l=1}^m  
    a_l \sum_{|\kappa|_1 = l} \frac{l!}{\kappa!} 
        \Brack{\prod_{j=1}^d \Brack{2\pi  z_j}^{2 \kappa_j} }   
        \Abs{\mc{F}\SqBrack{f}(z)}^2  }   \\ 
    &= \int dz \Abs{\mc{F}\SqBrack{f}(z)}^2 \SqBrack{ 1 + \sum_{l=1}^m  
    a_l \sum_{|\kappa|_1 = l} \frac{l!}{\kappa!} 
        \Brack{\prod_{j=1}^d \Brack{2\pi  z_j}^{2 \kappa_j} }   
          }   \\
    &= \int dz \Abs{\mc{F}\SqBrack{f}(z)}^2 \SqBrack{ 1 + \sum_{l=1}^m  
    a_l   \cdot (2\pi)^{2l} \sum_{|\kappa|_1 = l} \frac{l!}{\kappa!} 
        \prod_{j=1}^d z_j^{2 \kappa_j} 
          }    \\ 
    &= \int dz \Abs{\mc{F}\SqBrack{f}(z)}^2 \SqBrack{ 1 + \sum_{l=1}^m  
    a_l  \cdot (2\pi)^{2l} \norm{z}^{2l} 
          }             
\end{flalign}  

\end{proof}

Using the above Lemma, the derivation of the kerenl is standard. 
Suppose $k^{\bar{a}}$ is the kernel corresponding to $\norm{f}_{\bar{a}}$ on $\mcH^{\bar{a}}$. 
It remains to observe that by the reproducing property and by Lemma \ref{lem:norm_in_fourier}, for all 
$x \in \RR^d$
\begin{flalign}
f(x) &= \inner{f}{k^{\bar{a}}_x}_{\bar{a}}  \\ 
     &= \int_{\RR^d} dz \spaceo  \mc{F}[f](z) \overline{\mc{F}[k^{\bar{a}}_x](z)} v_{\bar{a}}^2(z).
\end{flalign}
On the other hand, by the Fourier inversion formula,  we  have 
\begin{flalign}
f(x) &=  \int dz \spaceo  \mc{F}[f](z) e^{2\pi i \inner{x}{z}}.
\end{flalign}
This implies that 
\begin{equation}
    \int dz \spaceo  \mc{F}[f](z) e^{2\pi i \inner{x}{z}} = 
    \int_{\RR^d} dz \spaceo  \mc{F}[f](z) \overline{\mc{F}[k^{\bar{a}}_x](z)} v_{\bar{a}}^2(z)
\end{equation}
holds for all $f \in \mcH^{\bar{a}}$, which by standard continuity considerations yields 
\begin{equation}
\label{eq:fourier_transform_of_k_a}
    \mc{F}[k^{\bar{a}}_x](z) = \frac{e^{-2\pi i \inner{x}{z}}}{v_{\bar{a}}^2(z)}.
\end{equation}
Using Fourier inversion again we obtain
\begin{equation}
k^{\bar{a}}(x,y) = \int_{\RR^d} \frac{e^{2\pi i \inner{y-x}{z}}}{v_{\bar{a}}^2(z)} dz = 
\int_{\RR^d} \frac{e^{2\pi i \inner{y-x}{z}}}{
     1 + \sum_{l=1}^m  
    a_l  \cdot (2\pi)^{2l} \norm{z}^{2l} 
} dz.
\end{equation}

\subsection{Sampling Approximation}
\label{sec:SDO_sampling_approximation_details}

As discussed in Section \ref{sec:sampling_kernel_v1}, we are interested in sampling points $z\in \RR^d$ from a finite non negative measure with density given by $w^a(z) = (1 +   a  \cdot (2\pi)^{2m} \norm{z}^{2m})^{-1}$. 
With a slight overload of notation, we will also denote by 
$w_a$ the scalar function $w_a: \RR \rightarrow \RR$, 
\begin{equation}
    w^a(r) = (1 +   a  \cdot (2\pi)^{2m} r^{2m})^{-1}.
\end{equation}

First, note that $w_a(z)$ depends on $z$ only through the norm $\norm{z}$, and thus a spherically symmetric  function. Therefore, with a spherical change of variables,  we can rewrite the integrals w.r.t $w_a^{-2}$ as follows: For any 
$f: \RR^d \rightarrow \CC$,
\begin{flalign}
    \int_{\RR^d} w_a(z) f(z) dz &= \int_{0}^\infty dr \int_{S^{d-1}} d\theta    \spaceo w_a(r) A_{d-1}(r) f(r \theta) \\
    &= 
    A_{d-1}(1) \int_{0}^\infty dr \int_{S^{d-1}} d\theta  \spaceo   \SqBrack{w_a(r) r^{d-1}} \cdot  f(r \theta).    \label{eq:w_a_sampling_full_expression}
\end{flalign}
Here $S^{d-1}$ the unit sphere in $\RR^d$, $\theta$ is sampled from the uniform probability measure on the sphere, $r$ is the radius, and 
\begin{equation}
    A_{d-1}(r) = \frac{2 \pi^{d/2}}{\Gamma(d/2)}r^{d-1} 
\end{equation}
is the $d-1$ dimensional volume of the sphere or radius $r$ in $\RR^d$. The meaning of \eqref{eq:w_a_sampling_full_expression} is that 
to sample from $w_a^{-2}$, we can sample $\theta$ uniformly from the sphere (easy), and $r$ from a density 
\begin{equation}
    \zeta(r) = w_a(r) r^{d-1}  = \frac{r^{d-1}}{1 +   a  \cdot (2\pi)^{2m} r^{2m}}
\end{equation}
on the real line. Note that the condition $m>d/2$ that we impose throughout is necessary. Indeed,  without this condition the decay of $\zeta(r)$ would not be fast enough at infinity, and the density would not have a finite mass.

As discussed in Section \ref{sec:sampling_kernel_v1}, $\zeta(r)$ is a density on a real line, with a single mode and an analytic expression, which allows easy computation of the derivatives. Such distributions can be efficiently sampled using, for instance,  off-the-shelf Hamiltonian Monte Carlo (HMC) samplers, \cite{betancourt2017conceptual}. In our experiments we have used an even simpler scheme, by discretizing $\RR$ into a grid of 10000 points, with limits wide enough to accommodate a wide range of parameters $a$.


\section{A Few Basic Properties of the Kernel}
\begin{prop} The kernel \eqref{eq:kernel_fourier_inv} is real valued and satisfies 
    \begin{equation}
        K^a(x,y) = \int_{\RR^d} \frac{\cos\Brack{2\pi \inner{y-x}{z}}}{1 +   
    a  \cdot (2\pi)^{2m} \norm{z}^{2m}} dz.
    \end{equation}    
\end{prop}
\begin{proof}
    Write $e^{2\pi i \inner{y-x}{z}} = cos(2\pi \inner{y-x}{z}) + i \sin(2\pi \inner{y-x}{z})$ 
    and observe that  $sin$ is odd in $z$, while $1 +   
    a  \cdot (2\pi)^{2m} \norm{z}^{2m}$  is even.
\end{proof}

\begin{prop} 
\label{prop:k_rescaling _property}
For all $x,y \in \RR^d$,
\begin{equation}
K^b(x,y) = b^{-\frac{d}{2m}} K^1(b^{-\frac{1}{2m}} x, b^{-\frac{1}{2m}} y) 
\end{equation}
\end{prop}
\begin{proof}
    Write $u = b^{\frac{1}{2m}}z$ and note that $du = (b^\frac{1}{2m})^d dz$. 
    We have 
\begin{flalign}
    K^b(x,y) &= \int_{\RR^d} \frac{\cos\Brack{2\pi \inner{y-x}{z}}}{
    1 + b \norm{2 \pi  \cdot z }^{2m} 
    } dz \\ 
    &= \int_{\RR^d} \frac{\cos\Brack{2\pi \inner{b^{-\frac{1}{2m}}(y-x)}{u}}}{
    1 + \norm{2 \pi u }^{2m} 
    } du  \cdot b^{-\frac{d}{2m}} \\ 
    &= 
    b^{-\frac{d}{2m}} K^1(b^{-\frac{1}{2m}} x, b^{-\frac{1}{2m}} y).
\end{flalign}
\end{proof}

\begin{lem}
\label{lem:k_x_a_l2_y_norm_expression}
There is a function $c(m)$ of $m$ such that for every $x \in \RR^d$,
\begin{equation}
    \int \Brack{K^a(x,y)}^2 dy = c(m) \cdot a^{-\frac{d}{2m}}.
\end{equation}
\end{lem}
\begin{proof}
Recall that for fixed $x$ and $a$, the Fourier transform satisfies 
$\mcF(k^a_x)(z) = \frac{e^{-2\pi i \inner{x}{z}}}{1+a (2\pi)^{2m}\norm{z}^{2m}}$, where $k^a_x(\cdot) = K^a(x,\cdot)$ (see eq. \eqref{eq:fourier_transform_of_k_a}).    
We thus have 
\begin{flalign}
    \norm{k^a_x}^2_{L_2} &= \norm{\mcF(k^a_x)}^2_{L_2} \\ 
    &= 
    \int \frac{e^{-2\pi i \inner{x}{z}} \cdot \overline{e^{-2\pi i \inner{x}{z}}}}{\Brack{1+a (2\pi)^{2m}\norm{z}^{2m}}^2}dz     \\ 
    &= \int \frac{1}{\Brack{1+a (2\pi)^{2m}\norm{z}^{2m}}^2}dz \\ 
    &= a^{-\frac{d}{2m}}\int \frac{1}{\Brack{1+ (2\pi)^{2m}\norm{z'}^{2m}}^2}dz', \\ 
\end{flalign}
where we have used the variable change  $z = a^{-\frac{1}{2m}} z'$.
\end{proof}

\begin{figure*}[t]
    \centering
    \includegraphics[scale=0.55]{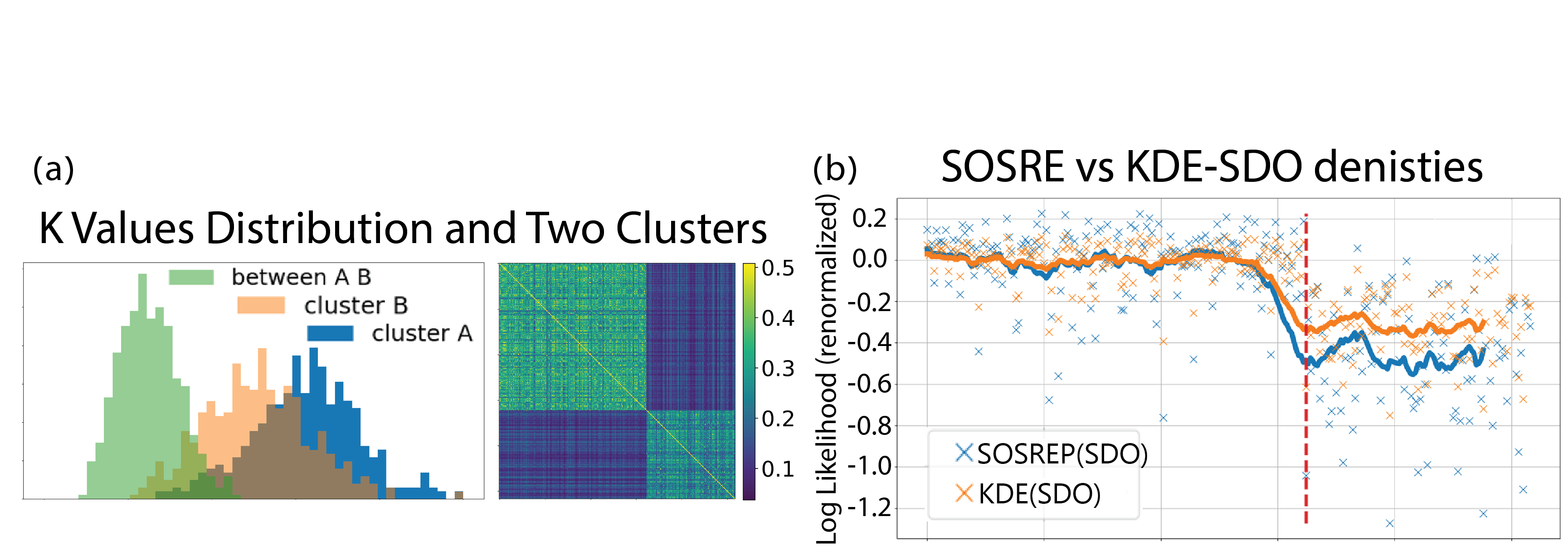}
    \caption{(a) Distribution of Kernel Values and SDO Kernel Values Inside and Between Clusters. (b) SOSREP and KDE Loglikelihoods. the x-axis represents points in the data, arranged by clusters,  y-axis shows the log-likelihood.}
    \label{fig:KDE_combined}
\end{figure*}

\section{Difference between SOSREP and KDE Models}
\label{sec:two_regions}

In this Section we construct an analytic example where the SOSREP estimator may differ arbitrarily from the KDE estimator with the same kernel. Thus, the models are not equivalent, and encode different prior assumptions. Briefly, we consider a block model, with two clusters. We'll show that in this particular setting, 
in KDE the clusters influence each other more strongly, i.e the points in one cluster contribute to the weight of the points in other cluster, yielding more uniform models. In contrast, in SOSREP, rather surprisingly, the density does not depend on the mutual position of the clusters (in a certain sense). Note that this is not a matter of \emph{bandwith} of the KDE, since both models use the same kernel. We believe that this property may explain the better performance of SOSREP in Anomaly Detection tasks, although further investigation would be required to verify this. 

Given a set of datapoints $S = \Set{x_i}$, for the purposes of this section the KDE estimator is the function 
\begin{equation}
\label{eq:kde_iner_comp_kde}
    f_{kde}(x) = f_{kde,S}(x) = \frac{1}{\Abs{S}} \sum_{i} k_{x_i}(x).
\end{equation}
Let $f_{SOSREP}$ be the solution of \eqref{eq:main_optimization_problem}.  
We will compare the ratios $f_{kde}(x_i)/f_{kde}(x_j)$ versus the corresponding quantities for SOSREP, $f^2_{SOSREP}(x_i)/f^2_{SOSREP}(x_j)$
for some pairs $x_i,x_j$. Note that these ratios do not depend on the normalization of $f_{kde}$ and $f^2_{SOSREP}$, and can be computed from the unnormalized versions. In particular, we do not require $k_{x_i}$ to be normalized in \eqref{eq:kde_iner_comp_kde}.

Consider a set $S$ with two components, $S = S_1 \cup S_2$,
with $S_1 = \Set{x_1,\ldots, x_N}$ and $S_2 = \Set{x'_1, \ldots x'_M}$ and with the following kernel values:
\begin{equation}
\label{eq:two_region_kernel} 
K =  \begin{cases}
        k(x_i,x_i) = k(x'_j,x'_j) = 1   & \text{ for all $i\leq N,j\leq M$} \\ 
        k(x_i,x_j) =  \gamma^2 & \text{for $i \neq j$}\\ 
        k(x'_i,x'_j) =  \gamma'^2 & \text{for $i \neq j$}\\ 
        k(x_i,x'_j) =  \beta \gamma \gamma' &   \text{for all $i,j$ }      
    \end{cases}
\end{equation}
This configuration of points is a block model with two components, or two clusters. The correlations between elements in the first cluster are $\gamma^2$, and are $\gamma'^2$ in the second cluster. Inter-cluster correlations are $\beta \gamma \gamma'$. We assume that $\gamma,\gamma,\beta \in [0,1]$ and w.l.o.g take $\gamma > \gamma'$.
While this is an idealized scenario to allow analytic computations, settings closely approximating the configuration \eqref{eq:two_region_kernel} often appear in real data. See Section \ref{sec:kde_iner_comparison_empirical_example} for an illustration. In particular, Figure \ref{fig:KDE_combined} show a two cluster configuration in that data, and the distribution of $k(x,x')$ values.

The KDE estimator for $K$ is simply  
\begin{equation}
\begin{split}
        f_{kde}(x_t) &= \frac{1}{N+M} \SqBrack{ 1 + (N-1) \gamma^2 + M \beta \gamma \gamma'} 
    \approx\\
    & \frac{N}{N+M} \gamma^2 + \frac{M}{N+M} \beta \gamma \gamma',
\end{split}
\end{equation}
for $x_t \in S_1$, where the second, approximate equality, holds for large $M,N$. To simplify the presentation, we shall use this approximation. However, all computations and conclusions also hold with the precise equality. For $x'_t \in S_2$ we similarly have 
   $ f_{kde}(x'_t)  \approx \frac{N}{N+M} \beta \gamma \gamma' + \frac{M}{N+M} \gamma'^2$,
and when $M=N$, the density ratio is 
\begin{equation}
\label{eq:kde_iner_kde_ratio}
    \frac{f_{kde}(x_t)}{f_{kde}(x'_t)} = \frac{\gamma^2 + \beta \gamma \gamma'}{\gamma'^2 + \beta \gamma \gamma'}.
\end{equation}

The derivation of the SOSREP estimator is considerably more involved. 
Here we sketch the argument, while full details are given in Supplementary Material Section \ref{sec:kde_vs_iner_comparison_proofs}.
First, recall from the previous section that the natural gradient in the $\alpha$ coordinates is given by $2\Brack{\beta - N^{-1} (K\beta)^{-1}}$. Since the optimizer of \eqref{eq:main_optimization_problem} must satisfy $\grad_f L =0$, we are looking for $\beta \in \RR^{N+M}$ such that $\beta = (K\beta)^{-1}$ (the term $N^{-1}$ can be accounted for by renormalization). Due to the symmetry of $K$ and since the minimizer is unique, we may take $\beta = (a,\ldots,a, b,\ldots,b)$, where $a$ is in first $N$ coordinates and $b$ is in the next $M$. 
Then $\beta = (K\beta)^{-1}$ is equivalent to $a,b$ solving the following system:
\begin{flalign}
\label{eq:ql_system_definition}
\begin{cases}
    a &= a^{-1} \SqBrack{1 + (N-1) \gamma^2} + b^{-1} M \beta \gamma \gamma' \\
    b &= a^{-1} N \beta \gamma \gamma'   + b^{-1} \SqBrack{1 + (M-1) \gamma'^2}      
\end{cases}
\end{flalign}
This is a non linear system in $a,b$. However, it turns out that it may be explicitly solved, up to a knowledge of a certain sign variable (see Proposition \ref{prop:two_variable_system}). Moreover, for $M=N$, the dependence on that sign variable vanishes, and we obtain
\begin{prop}
\label{cor:kde_iner_diff_cor}
Consider the kernel and point configuration described by \eqref{eq:two_region_kernel}, with $M=N$. Then for every $x_t\in S_1, x'_s \in S_2$, 
\begin{equation}
\label{eq:kde_iner_iner_ratio}
    \frac{f_{SOSREP}(x_t)}{f_{SOSREP}(x'_s)} = \frac{\gamma^2}{\gamma'^2}.
\end{equation}
In particular, the ratio does not depend on $\beta$.
\end{prop}
It remains to compare the ratio \eqref{eq:kde_iner_iner_ratio} to KDE's ratio \eqref{eq:kde_iner_kde_ratio}. If $\beta = 0$, when the clusters are maximally separated, the ratios coincide. However, let us consider the case, say, $\beta = \half$, and assume that $\gamma' \ll \gamma$. Then in the denominator of \eqref{eq:kde_iner_kde_ratio} the larger term is $\beta \gamma \gamma'$, which comes from the influence of the first cluster on the second. This makes the whole ratio to be of the order of a constant. On the other hand, in SOSREP there is no such influence, and the ratio \eqref{eq:kde_iner_iner_ratio} may be arbitrarily large. We thus expect the gap between the cluster densities to be larger for SOSREP, which is indeed the case empirically. One occurence of this on real data is illustrated in Figure \ref{fig:KDE_combined}.

\subsection{Evaluation of the Difference Between SOSREP and KDE for Real Data}
\label{sec:kde_iner_comparison_empirical_example}
We have performed spectral clustering of the ``letter'' dataset from ADBech (\cite{han2022adbench}), using the empirical SDO kernel as affinity matrix for both SOSREP and KDE. We then have chosen two clusters that most resemble the two block model \eqref{eq:two_region_kernel} in Section \ref{sec:two_regions}. The kernel values inside and between the clusters are shown in Figure \ref{fig:KDE_combined}a. 
Next, we train the SOSREP and KDE models for just these two clusters (to be compatible with the setting of Section \ref{sec:two_regions}. The results are similar for densities trained on full data). The log of these SOSREP and KDE densities in shown in Figure \ref{fig:KDE_combined}b (smoothed by running average). By adding an appropriate constant, we have arranged that the mean of both log densities is 0 on the first cluster. Then one can clearly see that the gap between the values on the first and second cluster is larger for the SOSREP model, yielding a less uniform model, as expected from the theory.

\section{KDE vs SOSREP Comparison Proofs}
\label{sec:kde_vs_iner_comparison_proofs}
In this section we develop the ingredients required to prove Proposition \ref{cor:kde_iner_diff_cor}. In section \ref{sec:two_block_solution} we reduce the solution of the SOSREP problem for the two block model to a solution of a non-linear system in two variables, and derive the solution of this system. 
In section \ref{sec:iner_two_block_prop_proof} we use these results to prove Proposition \ref{cor:kde_iner_diff_cor}.

\subsection{Solution Of SOSREP for a 2-Block Model}
\label{sec:two_block_solution}

As discussed in section \ref{sec:two_regions}, any SOSREP solution 
$f$ must be a zero point of the natural gradient, $\grad_f L = 0$. 
Using the expressions given following Lemma \ref{lem:gradients}, this implies $\beta = \frac{1}{N} (K\beta)^{-1}$. Since we are only interested in $f$ up to a scalar normalization, we can equivalently assume simply $\beta = (K\beta)^{-1}$. 
Further, 
by symmetry consideration we may take $\beta = (a,\ldots,a, b,\ldots,b)$, where $a$ is in first $N$ coordinates and $b$ is in the next $M$. Then, as mentioned in section \ref{sec:two_regions}, $\beta = (K\beta)^{-1}$ is equivalent to $a,b$ solving the following system:
\begin{flalign}
\label{eq:ql_system_definition_1}
\begin{cases}
    a &= a^{-1} \SqBrack{1 + (N-1) \gamma^2} + b^{-1} M \beta \gamma \gamma' \\
    b &= a^{-1} N \beta \gamma \gamma'   + b^{-1} \SqBrack{1 + (M-1) \gamma'^2}      
\end{cases}
\end{flalign}

It turns out that it is possible to derive an expression for the ratio of the squares of the solutions to this system in the general case. 

\begin{prop}[Two Variables SOSREP System]
\label{prop:two_variable_system}
Let $a,b$ be solutions of 
\begin{equation}
\label{eq:two_variables_system_general}
    \begin{cases}
        a &= H_{11} a^{-1} + H_{12} b^{-1}  \\
        b &= H_{21} a^{-1} + H_{22} b^{-1} 
    \end{cases}
\end{equation}
Then 
\begin{flalign}
\label{eq:two_variables_prop_result}
    a^2 / b^2 = H_{11}^2
    \Brack{     
    \frac{  -(H_{21} + H_{12})  - \rho
     \sqrt{(H_{21} - H_{12})^2 +4H_{11}H_{22}}}
    {2 H_{11}H_{22} + H_{12}\SqBrack{ -(H_{21} - H_{12}) + \rho
    \sqrt{(H_{21} - H_{12})^2 +4H_{11}H_{22}}}  
    }     
 }^2 
\end{flalign}
for a $\rho$ satisfying $\rho \in \Set{+1,-1}$.
\end{prop}

\begin{proof}
Write $u = a^{-1}$, $v=b^{-1}$, and multiply the first and second equations by $u$ and $v$ respectively. Then we have
\begin{equation}
    \begin{cases}
        1 &= H_{11} u^2 + H_{12} uv  \\
        1 &= H_{21} uv + H_{22} v^2. 
    \end{cases}
\end{equation}
We write 
\begin{equation}
\label{eq:v_via_u_expression}
    v = \Brack{1 - H_{11} u^2} / H_{12}u.
\end{equation}
Also, from the first equation, 
\begin{equation}
    H_{12} uv = 1 - H_{11} u^2.  
\end{equation}
Substituting into the second equation, 
\begin{equation}
    1 = \frac{H_{21}}{H_{12}} \Brack{1 - H_{11} u^2} 
    + H_{22} \frac{\Brack{1 - H_{11} u^2}^2}{\Brack{ H_{12}u }^2}.
\end{equation}
Finally setting $s = u^2$ and multiplying by $H_{12}^2 s$,
\begin{equation}
    H_{12}^2 s = H_{21} H_{12} s \Brack{1 - H_{11} s} 
    + H_{22} \Brack{1 - H_{11} s}^2.    
\end{equation}
Collecting terms, we have 
\begin{equation}
    s^2 (H_{11}^2 H_{22} - H_{11}H_{12}H_{21}) + s (H_{12}H_{21} - H_{12}^2 -2 H_{11}H_{22} ) + H_{22} = 0.    
\end{equation}
Solving this, we get 
\begin{equation}
    s = \frac{-(H_{12}H_{21} - H_{12}^2 -2 H_{11}H_{22} ) \pm \sqrt{(H_{12}H_{21} - H_{12}^2 -2 H_{11}H_{22} )^2 - 4(H_{11}^2 H_{22} - H_{11}H_{12}H_{21})H_{22}}}{2(H_{11}^2 H_{22} - H_{11}H_{12}H_{21})}.
\end{equation}
The expression inside the square root satisfies 
\begin{flalign}
    &(H_{12}H_{21} - H_{12}^2 -2 H_{11}H_{22} )^2 - 4(H_{11}^2 H_{22} - H_{11}H_{12}H_{21})H_{22} \\ 
    &= (H_{12}(H_{21} - H_{12}) -2 H_{11}H_{22} )^2 - 4(H_{11}^2 H_{22} - H_{11}H_{12}H_{21})H_{22} \\ 
    &= H_{12}^2(H_{21} - H_{12})^2 -4 H_{11}H_{22}H_{12}(H_{21} - H_{12}) 
    +4H_{11}H_{12}H_{21}H_{22} \\ 
    &= H_{12}^2(H_{21} - H_{12})^2 +4 H_{11}H_{22}H_{12}^2  \\ 
    &= H_{12}^2 \SqBrack{(H_{21} - H_{12})^2 +4H_{11}H_{22} }
\end{flalign}
Thus, simplifying, we have 
\begin{equation}
    u^2 = s =  \frac{-(H_{12}(H_{21} - H_{12}) -2 H_{11}H_{22} ) + \rho
    H_{12} \sqrt{(H_{21} - H_{12})^2 +4H_{11}H_{22}}}
    {2H_{11}(H_{11} H_{22} - H_{12}H_{21})}, 
\end{equation}
where $\rho \in \Set{+1, -1}$.

Rewriting \eqref{eq:v_via_u_expression} again, we have 
\begin{equation}
\label{eq:vsq_via_usq_expression}
    v^2 = \frac{\Brack{1 - H_{11} u^2}^2}{  H_{12}^2 u^2}.
\end{equation}

Further, 
\begin{flalign}
    a^2 / b^2 &= v^2 / u^2  = \frac{\Brack{1 - H_{11} u^2}^2}{  H_{12}^2 u^4} \\ 
    &= \Brack{ \frac{1 - H_{11} u^2}{  H_{12} u^2} }^2 \\
    &= \Brack{ \frac{1}{  H_{12} u^2} - \frac{H_{11}}{H_{12}} }^2 \\ 
    &= \Brack{\frac{H_{11}}{H_{12}} }^2
    \Brack{     
    \frac{2(H_{11} H_{22} - H_{12}H_{21})}{-(H_{12}(H_{21} - H_{12}) -2 H_{11}H_{22} ) + \rho
    H_{12} \sqrt{(H_{21} - H_{12})^2 +4H_{11}H_{22}}}     
- 1 }^2   \\ 
    &= \Brack{\frac{H_{11}}{H_{12}} }^2
    \Brack{     
    \frac{2(H_{11} H_{22} - H_{12}H_{21}) +(H_{12}(H_{21} - H_{12}) -2 H_{11}H_{22} ) - \rho
    H_{12} \sqrt{(H_{21} - H_{12})^2 +4H_{11}H_{22}}}
    {-(H_{12}(H_{21} - H_{12}) -2 H_{11}H_{22} ) + \rho
    H_{12} \sqrt{(H_{21} - H_{12})^2 +4H_{11}H_{22}}}     
 }^2   \\ 
 &=\Brack{\frac{H_{11}}{H_{12}} }^2
    \Brack{     
    \frac{ - 2H_{12}H_{21} +H_{12}(H_{21} - H_{12})  - \rho
    H_{12} \sqrt{(H_{21} - H_{12})^2 +4H_{11}H_{22}}}
    {-(H_{12}(H_{21} - H_{12}) -2 H_{11}H_{22} ) + \rho
    H_{12} \sqrt{(H_{21} - H_{12})^2 +4H_{11}H_{22}}}     
 }^2 \\ 
 &=H_{11}^2
    \Brack{     
    \frac{ - 2H_{21} +H_{21} - H_{12}  - \rho
     \sqrt{(H_{21} - H_{12})^2 +4H_{11}H_{22}}}
    {-(H_{12}(H_{21} - H_{12}) -2 H_{11}H_{22} ) + \rho
    H_{12} \sqrt{(H_{21} - H_{12})^2 +4H_{11}H_{22}}}     
 }^2 \\ 
&=H_{11}^2
    \Brack{     
    \frac{  -(H_{21} + H_{12})  - \rho
     \sqrt{(H_{21} - H_{12})^2 +4H_{11}H_{22}}}
    {-(H_{12}(H_{21} - H_{12}) -2 H_{11}H_{22} ) + \rho
    H_{12} \sqrt{(H_{21} - H_{12})^2 +4H_{11}H_{22}}}     
 }^2 \\ 
&=H_{11}^2
    \Brack{     
    \frac{  -(H_{21} + H_{12})  - \rho
     \sqrt{(H_{21} - H_{12})^2 +4H_{11}H_{22}}}
    {2 H_{11}H_{22} + H_{12}\SqBrack{ -(H_{21} - H_{12}) + \rho
    \sqrt{(H_{21} - H_{12})^2 +4H_{11}H_{22}}}  
    }     
 }^2 
\end{flalign}

\end{proof}

\subsection{Proof Of Proposition \ref{cor:kde_iner_diff_cor}}
\label{sec:iner_two_block_prop_proof}

Similarly to the case with KDE, we will use the following approximation of  the system \eqref{eq:ql_system_definition_1} 
\begin{flalign}
\label{eq:ql_system_approx}
\begin{cases}
    a &= a^{-1} N \gamma^2 + b^{-1} M \beta \gamma \gamma' \\
    b &= a^{-1} N \beta \gamma \gamma'   + b^{-1} M \gamma'^2      
\end{cases}
\end{flalign}

\begin{proof}
Let $f$ be the SOSREP solution.
By definition, the ratio $\frac{f(x_t)}{f(x'_s)}$ is given by 
$a^2 / b^2$ where $a,b$ are the solutions to  \eqref{eq:ql_system_approx}. That is, we take 
 $H_{12} = H_{21} =  \beta \gamma \gamma'$, $H_{11} = \gamma^2$, and $H_{22} = \gamma'^2$ in Proposition \ref{prop:two_variable_system}. Note that we have removed the dependence on $N$, 
 since it does not affect the ratio. 
By Proposition \ref{prop:two_variable_system}, substituting into \eqref{eq:two_variables_prop_result},
\begin{flalign}
    &H_{11}^2
    \Brack{     
    \frac{  -(H_{21} + H_{12})  - \rho
     \sqrt{(H_{21} - H_{12})^2 +4H_{11}H_{22}}}
    {2 H_{11}H_{22} + H_{12}\SqBrack{ -(H_{21} - H_{12}) + \rho
    \sqrt{(H_{21} - H_{12})^2 +4H_{11}H_{22}}}  
    } 
    }^2 \\  
    &= H_{11}^2
    \Brack{     
    \frac{  -H_{12}  - \rho
     \sqrt{H_{11}H_{22}}}
    { H_{11}H_{22} + H_{12} \rho
    \sqrt{H_{11}H_{22}}
    } 
    }^2 \\  
    &= 
    \gamma^4 
    \Brack{     
    \frac{  -2 \beta \gamma \gamma'  - 2\rho
     \gamma \gamma'}
    {2 \gamma^2 \gamma'^2 + 2 \beta \gamma \gamma'  \rho
    \gamma \gamma'  }
    }^2 \\ 
    &= 
    \Brack{\frac{\gamma^2 \gamma \gamma'}{\gamma^2 \gamma'^2}}^2
    \frac{  (\beta   + \rho)^2  }
    { (1  + \beta  \rho)^2}
\end{flalign}
It remains to note that 
$\frac{  (\beta   + \rho)^2  }
    { (1  + \beta  \rho)^2} = 1$ for any $\beta$ and $\rho \in \Set{+1,-1}$.
\end{proof}


\section{Invariance of $\mcC$ under Natural Gradient}
\label{sec:invariance_of_cone_under_nat}
Define the non-negative cone of functions $\mcC \subset \mcH$ by
\begin{equation}
    \mcC = \Set{f \in \mcH \setsep f(x) \geq 0 \spaceo \forall x\in \mcX}.
\end{equation}
As discussed in section \ref{sec:basic_framework},  the functional $L(f)$ is convex on $\mcC$. 

We now show that if the kernel $k$ is non-negative, then the cone $\mcC$ is invariant under the natural gradient steps. In particular, this means that if one starts with initialization in $\mcC$ (easy to achieve), then the optimization trajectory stays in $\mcC$, without a need for computationally heavy projection methods. Note that this is unlikely to be true for the standard gradient.  
Recall that the expression \eqref{eq:grad_alpha} for the natural gradient is given in Lemma \ref{lem:gradients}.
\begin{prop}
\label{cor:NgPosCone}
Assume that $k(x,x')\geq 0$ for all $x,x' \in \mcX$ and that  $\lambda < 0.5$. If $f \in \mcC$, then also 
$f' := f - 2 \lambda \SqBrack{ 
       f - \frac{1}{N} \sum_{i=1}^N f^{-1}(x_i) k_{x_i}
    } \in C
$.
\end{prop}
\begin{proof}
Indeed, by opening the brackets, 
\begin{align*}
f' = \left( 1 - 2 \lambda\right)f + 2 \lambda \SqBrack{ \frac{1}{N} \sum_{i=1}^N f^{-1}(x_i) k_{x_i}},   
\end{align*}
which is a non-negative combination of functions in $\mcC$, thus yielding the result.
\end{proof}

\section{Fisher Divergence for Hyper-Parameters Selection}\label{SEC:FD}
The handling of unnormalized models introduces a particular nuance in the context of hyperparameter tuning, as it prevents the use of the maximum likelihood of the data in order to establish the optimal parameter. When confronted with difficulties associated with normalization, it is common to resort to score-based methods. The score function is defined as 
\begin{equation}
    s(x; a) = \nabla_{x} \log p_m(x;a), 
\end{equation}
where $p_m(x;a)$ is a possibly unnormalized probability density on $\RR^d$, evaluated at $x\in \RR^d$, and dependent on the hyperparameter $a$. Since the normalization constant is independent of $x$, and $s$ is defined via the gradient in $x$,  $s$ is independent of the normalization.  As a result, distance metrics between distributions that are based on the score function, such as the Fisher Divergence, can be evaluated using non-normalized distributions. 

The Fisher Divergence (FD) is a similarity measure between distributions, which is based on the score function -- the gradient of the log likelihood. In particular, it does not require the normalization of the density. While the link between FD and maximum likelihood estimation has been previously studied in the context of score matching \cite{DBLP:journals/corr/abs-1205-2629}, to the best of our knowledge, this represents the first application of FD for hyperparameter tuning. 

The divergence between data and a model can be approximated via the methods of \cite{hyvarinen2005estimation}, which have been recently computationally improved in \cite{song2020sliced} in the context of score-based generative models, \cite{song2019generative}. In particular, we adapt the Hutchinson trace representation-based methods used in \cite{song2020sliced} and \cite{grathwohl2018ffjord} to the case of models of the form \eqref{eq:main_optimization_problem}.

In this work, we employ this concept, leveraging it to identify the choice of parameters (in our case, $a$, the smoothness parameter) that minimize the FD between the density learned on the training set and the density inferred on the test set.
Specifically, we apply \emph{score-matching} (\cite{hyvarinen2005estimation}), a particular approach to measuring the Fisher divergence between a dataset sampled from an unknown distribution and a proposed distribution model. 

\subsection{Score Matching and Fisher Divergence}
Given independent and identically distributed samples ${x_1, \ldots , x_N } \in \RR^D$ from a distribution $p_d(x)$ and an un-normalized density learned, $\tilde{p_m}(x; a)$ (where $a$ is a parameter). Score matching sets out to reduce the Fisher divergence between $p_d$ and $\tilde{p_m}(\cdot; a)$, formally expressed as $$L(a) = \frac{1}{2} \cdot E_{p_d}[\lVert s_m(x; a) − s_d(x) \rVert ^2]$$
As detailed in \cite{hyvarinen2005estimation}, the technique of integration by parts can derive an expression that does not depend on the unknown latent score function $s_d$: 
 $$L(a; {x_1,\ldots,x_n}) = \frac{1}{N} \sum_{i=1}^{N} \left[ tr( \nabla_x s_m(x_i; a )) + \frac{1}{2} \cdot \lVert s_m(x_i; a )\rVert ^2 \right]+ C $$
In this context, C is a constant independent of $a, tr(\cdot) $ denotes the trace of a matrix, and $\nabla_x s_m(x_i; a ) = \nabla_{x}^2 log(\tilde{p_m}(x_i; a))$ is the Hessian of the learned log-density function evaluated at $x_i$.

\subsection{The Hessian Estimation for Small $a$'s}
Deriving the Hessian for small $a$ values proves to be challenging.
Note that small $a$ values signify overfitting to the training data, consequently, this leads to a density that is mainly close to zero between samples, thereby making the process highly susceptible to significant errors in numerically calculating the derivatives. This situation results in a Hessian that is fraught with noise. Hence, our strategy focuses on locating a stable local minimum with the highest possible $a$. In this context, we define a stable local minimum as a point preceded and succeeded by three points, each greater than the focal point.

\subsection{Approximating the Hessian Trace}

Although this method holds promise, it's worth noting the computational burden tied to the calculation of the Hessian trace. To mitigate this, we rely on two techniques. First, we utilize Hutchinson’s trace estimator \citep{doi:10.1080/03610918908812806}, a robust estimator that facilitates the estimation of any matrix's trace through a double product with a random vector $\epsilon$:
$$Tr(H) = E_{\epsilon}\left[  \epsilon^T H \epsilon \right].$$
Here $\epsilon$ is any random vector on $\RR^d$ with mean zero and covariance $I$. This expression allows to reduce amount of computation of $Tr(H)$, by computing the products $H\epsilon$ directly, for a few samples of $\epsilon$, without the need to compute the full $H$ itself. A similar strategy has been recently employed in  \cite{grathwohl2018ffjord} in a different context, for a trace computation of a Jacobian of a density transformation, instead of the score itself.   

In more detail, score computations can be performed efficiently and in a 'lazy' manner using automatic differentiation, offered in frameworks such as PyTorch. This allows us to compute a vector-Hessian product $H\epsilon$ per sample without having to calculate the entire Hessian for all samples, a tensor of dimensions $N\times (d\times d)$, in advance. More specifically, we utilize PyTorch's automatic differentiation for computing the score function, which is a matrix of $N\times d$. Subsequently, this is multiplied by $\epsilon$. We then proceed with a straightforward differentiation $\nabla_x s(x_i) = \frac{1}{h} \cdot \left( s \left(x_i + h\cdot \epsilon \right) - s \left(x_i \right) \right)$ for small step $h$, followed by a summation which is lazily calculated through PyTorch (see Algorithm \ref{algo:fastHes}). 

\begin{algorithm}
\caption{Calculating Hutchinson's Trace Estimator}
\label{algo:fastHes}
\begin{algorithmic}[1]
    \REQUIRE Score function $s$, small constant $h$, sample $x$, \# of random vectors $n$
    \STATE Initialize $traceEstimator$ to 0
    \FOR{$i=1$ to $n$}
        \STATE Sample random vector $\epsilon$ from normal distribution
        \STATE Calculate $s(a; x + h*\epsilon)$
        \STATE Calculate $(s(a; x + h*\epsilon) - s(a; x))$
        \STATE Compute $(1/h) \cdot (s(a; x + h*\epsilon) - s(a; x)) \cdot \epsilon$
        \STATE Add result to $traceEstimator$
    \ENDFOR
    \STATE Return $\frac{traceEstimator}{n}$
\end{algorithmic}
\end{algorithm}

\section{Experiments}

\subsection{AUC-ROC Performance Analysis}
\label{sec:supp_aucroc}
\begin{figure*}[t]
    \centering
    \includegraphics[scale =0.3]{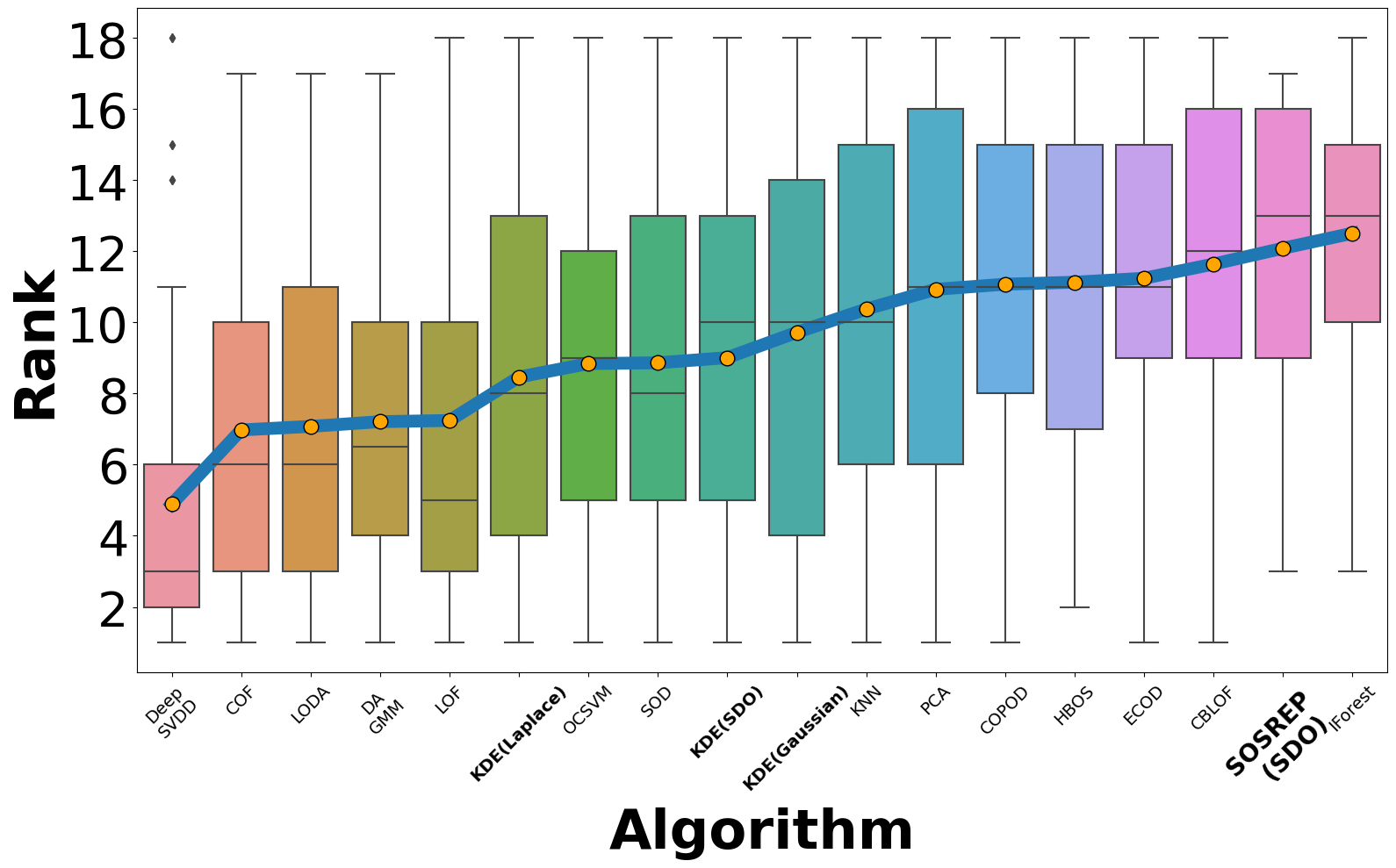}
    \caption{Anomaly Detection Results on ADBench. Relative Ranking Per Dataset, Higher is Better. SOSREP is Second Best Among 18 Algorithms}
    \vspace{-5.00mm} 
    \label{fig:ADbench_ranks}
\end{figure*}

In addition to AUC-ROC, we also focus on a ranking system as follows: for each dataset, we convert raw AUC-ROC scores of the methods into rankings from 1 to 18. Here, 18 denotes the best performance on a given dataset. This mitigates bias inherent in averaging AUC-ROC scores themselves across datasets, due to generally higher AUC-ROC scores on easier datasets. This is important since no single AD method consistently outperforms others in all situations, as discussed in detail in \cite{han2022adbench}.
In Figure \ref{fig:ADbench_ranks}, for each algorithm, we present a box plot where the average AUC-ROC ranking across all datasets is indicated by an orange dot. Additionally, the boundaries of the box represent the 25th and 75th quantiles, while the line inside the box denotes the median. The algorithms are sorted by the average AUC-ROC ranking. 

Figure \ref{fig:supp:heatmap} shows a heatmap representation of the AUC-ROC values. In this visualization, the size of the circle symbolizes the corresponding AUC value, while the color gradient signifies the deviation in AUC value from SOSREP. The purpose of this heatmap is to offer a graphical interpretation of the AUC-ROC performance levels, demonstrating how they diverge from the performance of SOSREP.

\begin{figure}
    \centering
    \includegraphics[width=0.7\linewidth]{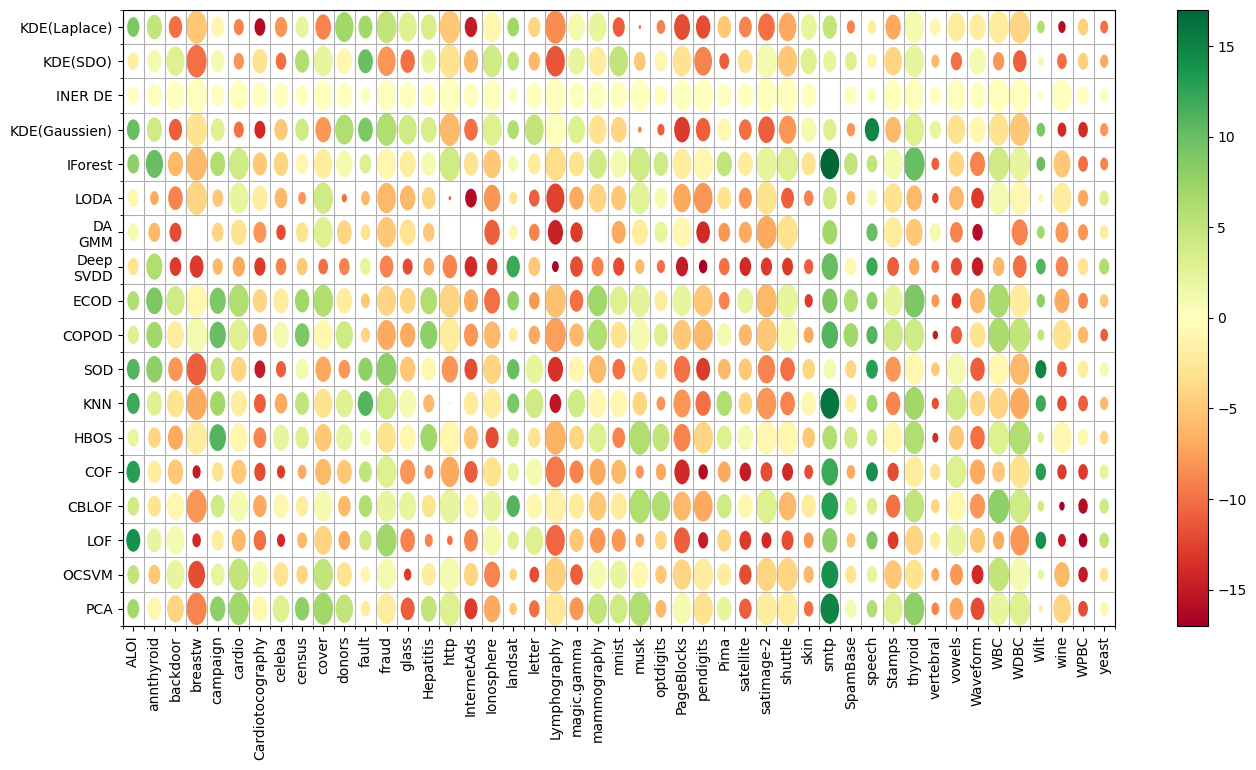}
    \caption{Heatmap of AUC-ROC values. Circles size represents absolute value, and the color the shift from SOSREP.}
    \label{fig:supp:heatmap}
\end{figure}

\subsection{Duplicate Anomalies.}
\label{sec:supp_duplicate_anomalies}

Duplicate anomalies are often encountered in various applications due to factors such as recording errors \cite{8416441}, a circumstance termed as "anomaly masking" \citep{10.1007/s10618-015-0444-8,10.5555/3045390.3045676}, posing significant hurdles for diverse AD algorithms. The significance of this factor is underscored in ADbench \citep{han2022adbench}, where duplicate anomalies are regarded as the most difficult setup for anomaly detection, thereby attracting considerable attention. To replicate this scenario, ADbench duplicates anomalies up to six times within training and test sets, subsequently examining the ensuing shifts in AD algorithms' performance across the 47 datasets discussed in our work.

As shown in ADBench, unsupervised methods are considerably susceptible to repetitive anomalies. In particular, as shown in Fig. 7a in the main text there, performance degradation is directly proportional to the increase in anomaly duplication. With anomalies duplicated six times, unsupervised methods record a median AUC-ROC reduction of -16.43$\%$, where in SOSREP the drop is less then 2$\%$. This universal performance deterioration can be attributed to two factors: 1) The inherent assumption made in these methods that anomalies are a minority in the dataset, a presumption crucial for detection. The violation of this belief due to the escalation in duplicated anomalies triggers the noticed performance downturn. 2) Methods based on nearest neighbours assume that anomalies significantly deviate from the norm (known as "Point anomalies"). However, anomaly duplication mitigates this deviation, rendering the anomaly less distinguishable.
Notice that while the first factor is less of a problem for a density based AD algorithm ( any time the anomalies are still not the major part of the data), the second factor could harmful to DE based AD algorithms as well.
The evidence of SOSREP possessing the highest median and its robustness to duplicate anomalies, along with a probable resistance to a high number of anomalies, not only emphasizes its superiority as a DE method but also underscores its potential to serve as a state-of-the-art AD method.

\subsection{Hyperparameter Tuning with On-the-fly Fisher-Divergence Minimization}\label{sec:supp:HP tuning FD}
A primary task at hand involves hyperparameter tuning to select the optimal $a$ according to the Fisher-Divergence (FD) procedure, as detailed in section \ref{SEC:FD} and Section \ref{sec:experiments_AD}. 
For efficient computation, we employ an on-the-fly calculation approach. The process initiates with a $a$ value that corresponds to the 'order'-th largest tested point. Subsequently, we explore one larger and one smaller $a$ value. If both these points exceed the currently tested point, we continue sampling. Otherwise, we shift the tested point to a smaller value. To avoid redundant calculations, the results are continually stored, ensuring each result is computed only once. This methodology provides a balance between computational efficiency and thorough exploration of the hyperparameter space.

\section{Algorithm Overview}
\label{Sec:ALGO}
In this section, we outline the procedure for our density estimation algorithm. Let $X \in \mathbb{R}^{N \times d}$ represent the training set and $Y \in \mathbb{R}^{N \times d}$ denote the test set for which we aim to compute the density. Initially, we establish the following hyper-parameters: 1.$N_z$, the number of samples $Z$ taken for the kernel estimation. (Notice $N_z$ is $T$ from \eqref{eq:kernel_sample_approx}) 2.$n_{iters}$, the number of iterations used in the gradient decent approximating the optimization for $\bm{\alpha}$.
3.$lr$, the learning rate applied in the gradient decent $\bm{\alpha}$ optimization process. 4.$n_{fd\_iters}$, the number of iterations for the Hessian trace approximation. 5.$h$, the step size employed for the Hessian trace.
Then, we follow Algorithm \ref{Algo:main}.

\begin{algorithm}
\caption{Density Estimation Procedure with Hypeparameter Tuning}
\label{Algo:main}
\begin{algorithmic}[1]
\REPEAT
\STATE Estimate the SDO kernel matrix via sampling as described in Section \ref{sec:INER_density_estimator} and detailed in Algorithm \ref{Algo:K}.
\STATE Determine the optimal $\bm{\alpha}$ as described in Section \ref{sec:INER_density_estimator} and detailed in Algorithm \ref{Algo:Oa}. $optimal\_\alpha$ forms $f := f^* \triangleq f_{optimal\_\alpha}$ and the density estimator is $F^2 \triangleq (f)^2 $.
\STATE Compute $F^2(Y)$ the density over $Y$ as described in Section \ref{sec:INER_density_estimator} and detailed in Algorithm \ref{Algo:F2}.
\STATE Assess the Fisher-divergence as described in Supplementary Material Section \ref{SEC:FD} and detailed in both Algorithm \ref{Algo:FD} and Algorithm \ref{algo:fastHes}.
\UNTIL{For each smoothness parameter $a$.}
\STATE The density estimator $F^2$ corresponds to the $a$ value that yields the minimal FD.
\end{algorithmic}
\end{algorithm}

\begin{algorithm}
\caption{$K(\cdot) \rightarrow \mathbb{R}^{2 \times N}$ : Sampling the Multidimensional SDO Kernel}
\label{Algo:K}
\begin{algorithmic}[1]
\REQUIRE $X, Y, a$

\STATE $\bm{\theta} \leftarrow \frac{\bm{g}}{\| \bm{g} \|_2} \text{ ; } \bm{g} \sim \mathcal{N}(\bm{0},\bm{1})$
\STATE $\bm{r} \sim \frac{r^{d-1}}{1 + a (2\pi r)^{2m}}$ using grid search.
\STATE $Z \leftarrow \bm{r} \cdot  \bm{\theta}^T$
\STATE $\bm{b} \sim  U[0,2\pi]$
\STATE $\text{samples} \leftarrow \frac{1}{N_z} \cdot \cos(X \cdot Z + b) \cdot \left( \cos(Y \cdot Z + b) \right)^T$
\end{algorithmic}
\end{algorithm}

\begin{algorithm}
\caption{$Optimal\_alpha(\cdot) \rightarrow \mathbb{R}^N$ : Calculating the Optimal Alphas}
\label{Algo:Oa}
\begin{algorithmic}[1]
\REQUIRE $X, \text{lr}, n_{\text{iters}}$

\STATE $K \leftarrow K(X,X)$
\STATE $\alpha \leftarrow [|\alpha_0|, \ldots, |\alpha_N|] : \alpha_i \sim \mathcal{N}(0,1)$
\FOR{$i = 1 \text{ in } n_{\text{iters}}$}
\STATE $\alpha \leftarrow \alpha - 2 \cdot lr \cdot ((K \cdot \alpha) - (K \cdot (1. / (K \cdot \alpha))) / N_{\text{data}})$
\ENDFOR
\end{algorithmic}
\end{algorithm}

\begin{algorithm}
\caption{$F^2 (\cdot) \rightarrow \mathbb{R}^{N_Y}$ : Density over coordinates $Y$ given observations $X$. }
\label{Algo:F2}
\begin{algorithmic}[1]
\REQUIRE $X, Y$

\STATE $\alpha \leftarrow Optimal\_alpha(X, \ldots)$
\STATE $e \leftarrow K(X, Y)$
\STATE $density \leftarrow (e \cdot \alpha)^2$
\end{algorithmic}
\end{algorithm}

\begin{algorithm}
\caption{$FD(\cdot) \rightarrow \mathbb{R}$ : Fisher Divergence with Hessian Trace Approximation.}
\label{Algo:FD}
\begin{algorithmic}[1]
\REQUIRE $X \text{(Train set)}, Y \text{(Test set)}, n_{fd\_iters}, h$
\STATE $scores \leftarrow \nabla \sum \log(F^2(X,Y))$
\STATE $\text{traces\_sum} \leftarrow \mathbf{0}$
\FOR{$i = 1 \text{ in } n_{\text{iters}}$}
	\STATE $\epsilon \sim \mathcal{U}(\{-1,0,1\}^{\text{size}(X)})$
	\STATE $\text{shifted\_scores} \leftarrow \nabla \sum \log(F^2(X,Y + h \cdot \epsilon))$
	\STATE $\text{traces\_sum} \leftarrow \text{traces\_sum} + \sum((\text{shifted\_scores} - scores) \cdot \epsilon)/ h$
\ENDFOR
\STATE $\text{traces} \leftarrow \text{traces\_sum}/n_{\text{iters}}$
\STATE $FD \leftarrow \mathbb{E}[\text{traces} + \frac{1}{2} \cdot \| \mathbf{scores} \|^2]$
\end{algorithmic}
\end{algorithm}

\section{Figure \ref{fig:pos_neg_frac} - List of Datasets }
\label{sec:datasets_ng_frc}
Those are the datasets for which we compared the fractions of negative values for $alpha$ optimization vs natural gradients presented in Figure \ref{fig:pos_neg_frac}, ordered according to the X-axis in the figure :
\begin{enumerate}
    \item Mnist
    \item Shuttle
    \item PageBlocks
    \item Mammography
    \item Magic.gamma
    \item Skin
    \item Backdoor
    \item Glass
    \item Lymphography
    \item Stamps
    \item WDBC
    \item SpamBase
    \item Hepatites
    \item Wine
    \item Letter

\end{enumerate}

\newpage
\section{Consistency Theorem}
\label{sec:consistency_proof}
In this Section we state and prove the consistency result for the SOSREP estimators. 

Recall that for any $a>0$, $\mcH^a$ is the RKHS with the norm given by  \eqref{eq:norm_definition_formal} and the kernel $k^a(x,y)$ given by \eqref{eq:kernel_fourier_inv}.

For any $f\in \mcH^1$ define the SOSREP loss 
\begin{equation}
\label{eq:L_f_definition_proof_of_consistency}
L(f) = L^a(f) = \half \Brack{-\frac{1}{N} \sum \log f^2(x_i) + \norm{f}^2_{\mcH^a}}.
\end{equation}
Note that $f\in \mcH^1$ if and only if $f\in \mcH^a$ for every $a>0$.
Recall from the  discussion in Section \ref{sec:gradients}
and that $L$ is convex when restricted to the open cone 
\begin{equation}
\label{eq:cone_c_prime_definition}
    \mcC' = (\mcC^a)' = \Set{ f \in span\Set{k_{x_i}}_1^N \setsep f(x_i) >0 }.
\end{equation}
Note that $\mcC'$ depends on $a$ since the kernel $k_{x_i} = k_{x_i}^a$ depends on $a$. Observe also that compared to Section \ref{sec:gradients}, here we require only positivity on the data points $x_i$ rather than on all $x\in \RR^d$, and we restrict the cone to the span of $x_i$ since all the solutions will be obtained there in any case. This of course does not affect the convexity. 

The consistency result we prove is as follows: 
\begin{thm}
\label{thm:consistency_supp}
Let $x_1,\ldots, x_N$ be i.i.d samples from a compactly supported density $v^2$ on $\RR^d$, such that $v\in \mcH^1$.  Set $a = a(N) = 1/N$, and let 
$u_N = u(x_1, \ldots, x_N; a(N))$ be the minimizer of the objective  
\eqref{eq:L_f_definition_proof_of_consistency} in the cone \eqref{eq:cone_c_prime_definition}. Then $\norm{u_N - v}_{L_2}$ converges to $0$ in probability. 
\end{thm}
In words, when $N$ grows, and the regularisation  size $a(N)$ decays as $1/N$, the the SOSREP estimators $u_N$ converge to $v$ in $L_2$.

As discussed in the main text, note also that since $\norm{v}_{L_2} = 1$ (as its a density), and since $\norm{u_N - v}_{L_2} \rightarrow 0$, it follows by the triangle inequality that $\norm{u_N}_{L_2} \rightarrow 1$. That is, the estimator $u_N$ becomes approximately normalized as $N$ grows.

In Section \ref{sec:overview_of_proof} below we provide an overview of the proof, and in Section \ref{sec:proof_details} full details are provided.

\subsection{Overview Of The Proof}
\label{sec:overview_of_proof}

As discussed in Section \ref{sec:literature}, consistency for the 1 dimensional case was shown in \cite{klonias1984class} and  our approach here follows similar general lines. The differences between the arguments are due to the more general multi dimensional setting here, and due to some difference in assumptions.

To simplify the notation in what follows we set  $\norm{\cdot}_a := \norm{\cdot}_{\mcH^a}, \inner{\cdot}{\cdot}_{a} := \inner{\cdot}{\cdot}_{\mcH^a}$. Recall that $k^a(x,y)$ denotes the kernel corresponding to $\mcH^a$,

The first step of the proof is essentially a stability result for the optimization of  \eqref{eq:L_f_definition_proof_of_consistency}, given by Lemma \ref{lem:strong_convexity_of_L} below. In this Lemma we observe that 
$L$ is strongly convex and thus for any function $v$, we have 
\begin{equation}
\norm{u-v}_{\mcH^a} \leq \norm{\grad L (v) }_{\mcH^a},      
\end{equation}
where $u$ is the true minimizer of $L$ in $\mcC'$ (i.e. the SOSREP estimator). This is particular means that if one can show that the right hand side above is small for the true density $v$, then the solution $u$ must be close to $v$. Thus we can concentrate on analyzing the simpler expression $\norm{\grad L (v) }_{\mcH^a}$ rather than working with the optimizer $u$ directly. Remarkably, this result is a pure Hilbert space result, and it holds for any kernel.

We now thus turn to the task of bounding $\norm{\grad L (v) }_{\mcH^a}$. As discussed in Section \ref{sec:gradients}, the gradient of $L$ in $\mcH^a$ is given by 
\begin{equation}
\grad L(f) = -\frac{1}{N} \sum f(x_i)^{-1} k_{x_i} + f.     
\end{equation}
Opening the brackets in $\norm{\grad L (v) }_{\mcH^a}^2$ we have 
\begin{flalign}
\norm{\grad L(v)}_a^2 &=  
\frac{1}{N^2}\sum_{i,j} v^{-1}(x_i) v^{-1}(x_j) k(x_i,x_j) 
- 2\frac{1}{N}\sum_{i} v^{-1}(x_i) \inner{k_{x_i}}{v}_a + \norm{v}^2_a
\\ 
&= 
\frac{1}{N^2}\sum_{i,j} v^{-1}(x_i) v^{-1}(x_j) k(x_i,x_j) 
- 2 + \norm{v}^2_a. \label{eq:sketch_line_grad}
\end{flalign}
By definitions, we clearly have $\norm{v}^2_a-1 \rightarrow 0$ when $a \rightarrow 0$. Thus we have to show that the first term 
in \eqref{eq:sketch_line_grad} concentrates around 1.  
To this end, we will first show that the expectation of 
\begin{equation}
\label{eq:sketch_grad_expression}
\frac{1}{N^2}\sum_{i,j} v^{-1}(x_i) v^{-1}(x_j) k(x_i,x_j)    
\end{equation}
(with respect to $x_i$'s) converges to 1. This is really the heart of the proof, as here we show why $v$ approximately minimizes the SOSREP objective, in expectation, by exploiting the interplay between the kernel, the regularizing coefficient, and the density $v$. 
This argument is carried out in Lemmas \ref{lem:convolution_approximation_to_identity}, \ref{lem:Y_ii_convergence_bound}, and \ref{lem:Y_ij_expectation_convergence}.

Once the expectation is understood, we simply use the Chebyshev inequality to control the deviation of \eqref{eq:sketch_grad_expression} around its mean. This requires the control of cross products of the terms in \eqref{eq:sketch_grad_expression} and can be achieved by arguments similar to those used to control the expectation. This analysis is carried out in Propositions 
\ref{prop:Y_ij_moment_bound} - \ref{prop:four_disjoint_indices}, 
 and Lemma \ref{lem:Y_ij_sum_second_moment_bound}.

One of the technically subtle issues throughout the proof is the 
presence of the terms $v^{-1}(x_i)$, due to which some higher moments and expectations are infinite. This prevents the use of standard concentration results and requires careful analysis. This is also the reason why obtaining convergence rates and high probability bounds is difficult, although we believe this is possible.

\subsection{Full Proof Details}
\label{sec:proof_details}

Throughout this section let $L_2$ be the $L_2$ space of the the Lebesgue measure on $\RR^d$, 
$L_2 = \Set{f :\RR^d \rightarrow \CC \setsep \int_{\RR^d} \Abs{f}^2 dx \leq \infty}$.

Next, we observe that $L(f)$ is \emph{strongly} convex with respect to the norm $\norm{\cdot}_{\mcH^a}$ (see \cite{nesterov2003introductory} for an introduction to strong convexity). As a consequence, we have the following: 
\begin{lem}
\label{lem:strong_convexity_of_L}
Let $u$ be the minimizer of $L$ in $\mcC'$. Then for every $v\in \mcC'$, 
\begin{equation}
\norm{u-v}_{\mcH^a} \leq \norm{\grad L (v) }_{\mcH^a}.     
\end{equation}
\end{lem}
\begin{proof}
The function $f \mapsto \norm{f}^2_{\mcH^a}$ is 2-strongly convex with respect to $\norm{\cdot }^2_{\mcH^a}$. Since $L(f)$ is obtained by adding a convex function, and multiplying by $\half$, it follows that $L(f)$ is 1-strongly convex. Strong convexity implies that for all $u,v\in \mcC'$ we have 
\begin{equation}
    \inner{\grad L(v) - \grad L (u)}{v-u}_a \geq \norm{v-u}_a^2,
\end{equation}
see \cite{nesterov2003introductory}, Theorem 2.1.9. 
Using the Cauchy Schwartz inequality and the fact that $u$ is a local minimum with $\grad L(u) = 0$, we obtain
\begin{flalign}
    \norm{\grad L(v)}_a \cdot \norm{u-v}_a \geq 
    \inner{\grad L(v) }{v-u}_a \geq \norm{v-u}_a^2,
\end{flalign}
yielding the result. 
\end{proof}

Now, suppose the samples $x_i$ are generated from a true density $(f^*)^2$. Let $v := f^*$ be the square root of this density. Then, to show that the estimator $u$ is close to $v$, it is sufficient to show that $\norm{\grad L(v)}_a$ is small. We write $\grad L(v)$ explicitly: 

\begin{flalign}
\norm{\grad L(v)}_a^2 &=  
\frac{1}{N^2}\sum_{i,j} v^{-1}(x_i) v^{-1}(x_j) k(x_i,x_j) 
- 2\frac{1}{N}\sum_{i} v^{-1}(x_i) \inner{k_{x_i}}{v}_a + \norm{v}^2_a
\\ 
&= 
\frac{1}{N^2}\sum_{i,j} v^{-1}(x_i) v^{-1}(x_j) k(x_i,x_j) 
- 2 + \norm{v}^2_a.  \label{eq:L_gradient_explicit_form}
\end{flalign}

Since $v$ is a fixed function in $\mcH^1$, it is clear from definition \eqref{eq:norm_definition_formal} that $\norm{v}^2_a - 1 \rightarrow 0$ as $a \rightarrow 0$. Thus, to bound $\norm{\grad L(v)}_a^2$, it suffices to bound 
\begin{equation}
\label{eq:v_inv_sum_proof}
\frac{1}{N^2} \sum_{i,j} v^{-1}(x_i) v^{-1}(x_j) k(x_i,x_j) - 1
\end{equation}
with high probability over $x_i$, when $N$ is large and $a$ is small. 

Note that the form \eqref{eq:kernel_fourier_inv} of the kernel $k^a$ implies that this is a stationary kernel, i.e. $k^a(x,y) = g^a(x-y)$ with 
\begin{equation}
\label{eq:g_a_definition_proof}
g^a(x) = \int_{\RR^d} \frac{e^{2\pi i \inner{-x}{z}}}{
     1 +   
    a  \cdot (2\pi)^{2m} \norm{z}^{2m} 
} dz.    
\end{equation}

\begin{lem}
\label{lem:convolution_approximation_to_identity}
Let $k^a(x,y)$ be the SDO kernel defined by \eqref{eq:kernel_fourier_inv}. 
For any function $v \in L_2$, and $x\in \RR^d$ set 
\begin{equation}
\label{eq:K_a_convo_def}
    (K_av)(x) = \int k^a(x,y) v(y) dy. 
\end{equation}
Then $K_a$ is a bounded operator from $L_2$ to $L_2$, with $\norm{K_a}_{op} \leq 1$ for every $a>0$. Moreover, for every $v\in L_2$, $\norm{K_a v - v}_{L_2} \rightarrow 0$ with $a \rightarrow 0$.
\end{lem}
\begin{proof}
Note first that $K_a$, given by \eqref{eq:K_a_convo_def}, is a convolution operator, i.e. $K_a v = g^a * v = \int g^a(x-y) v(y) dy$.
Further, by the Fourier inversion formula, the Fourier transform of $g^a$ satisfies 
$\mcF g^a (z) = \frac{1}{1 + a  \cdot (2\pi)^{2m} \norm{z}^{2m}}$ 
(see Section \ref{sec:full_kernel_derivation_proof} for further details), and recall that $\mcF (g^a * v) = \mcF g^a \cdot \mcF v$.  
Since $\Abs{\mcF g^a(z)} \leq 1$ for every $z$ and $a>0$, this implies in particular that $K_a$ has operator norm at most 1. Next, by the Plancharel equality we have 
\begin{flalign}
\norm{K_a v - v}^2_{L_2} &= 
\norm{\mcF(K_a v - v)}^2_{L_2}  \\ 
&= \norm{\mcF(v) \cdot \Brack{\frac{1}{1 + a  \cdot (2\pi)^{2m} \norm{z}^{2m}}-1}}^2_{L_2}  \\ 
&= \norm{\mcF(v) \cdot \Brack{\frac{a  \cdot (2\pi)^{2m} \norm{z}^{2m}}{1 + a  \cdot (2\pi)^{2m} \norm{z}^{2m}}}}^2_{L_2} \\ 
&= \int \Abs{\mcF(v)(z)}^2 \cdot \Brack{\frac{a  \cdot (2\pi)^{2m} \norm{z}^{2m}}{1 + a  \cdot (2\pi)^{2m} \norm{z}^{2m}}}^2 dz 
\label{eq:convolution_approx_proof}
\end{flalign}
Fix $\eps >0$. For a radius $r$ denote by $B(r) = \Set{z \setsep \norm{z}\leq r}$ the ball of radius $r$, and let $B^c(r)$ be its complement. 
Since $\int \Abs{\mcF(v)(z)}^2 dz < \infty$, there is 
$r >0$ large enough such that $\int_{B^c(r)} \Abs{\mcF(v)(z)}^2 dz \leq \eps$. We bound \eqref{eq:convolution_approx_proof} on $B(r)$ and $B^c(r)$ separately. 
Choose $a>0$ such that $a \leq \Brack{(2\pi)^{2m} r^{2m}}^{-1} \cdot \eps^{\half} \cdot \norm{v}_{L_2}^{-1}$.
Then 
\begin{flalign}
&\int \Abs{\mcF(v)(z)}^2 \cdot \Brack{\frac{a  \cdot (2\pi)^{2m} \norm{z}^{2m}}{1 + a  \cdot (2\pi)^{2m} \norm{z}^{2m}}}^2 dz    \\ 
&= \int_{B(r)} \Abs{\mcF(v)(z)}^2 \cdot \Brack{\frac{a  \cdot (2\pi)^{2m} \norm{z}^{2m}}{1 + a  \cdot (2\pi)^{2m} \norm{z}^{2m}}}^2 dz    
+
\int_{B^c(r)} \Abs{\mcF(v)(z)}^2 \cdot \Brack{\frac{a  \cdot (2\pi)^{2m} \norm{z}^{2m}}{1 + a  \cdot (2\pi)^{2m} \norm{z}^{2m}}}^2 dz    \\ 
&\leq \eps \cdot \norm{v}_{L_2}^{-1} 
\int_{B(r)} \Abs{\mcF(v)(z)}^2  dz  
+ 
\int_{B^c(r)} \Abs{\mcF(v)(z)}^2 dz  \\ 
&\leq \eps + \eps.
\end{flalign}
This completes the proof.
\end{proof}

\begin{assume}
Assume that the density $v^2$ is compactly supported in $\RR^d$, 
and denote the support by $B = supp(v^2)$. 
\end{assume}
Note in particular that this implies that $B$ is of finite Lebesgue measure, $\lambda(B) \leq \infty$.

We treat the components with $i\neq j$ and $i=j$ in the sum \eqref{eq:v_inv_sum_proof} separately. 
In the case $i=j$ set $Y_i = v^{-2}(x_i) k(x_i,x_i) = g_a(0) v^{-2}(x_i)$, where $g^a$ was defined in \eqref{eq:g_a_definition_proof}. 
Note that the variable $Y_i$ has an expectation, but does not necessarily have higher moments. Nevertheless, it is still possible to bound the sum using the Marcinkiewicz-Kolmogorov strong law of large numbers, see \cite{loeve1977elementary}, Section 17.4, point $4^{\circ}$.  We record it in the following Lemma, where we also allow $a$ to depend on $N$, denoting $a = a(N)$, and assume that $a(N)$ does not decay too fast with $N$.

\begin{lem} 
\label{lem:Y_ii_convergence_bound}
Assume that  
$\lim_{N\rightarrow \infty} a^{-\frac{d}{2m}}(N) /  N = 0$. Then 
\begin{equation}
     \frac{g^{a(N)}(0)}{N^2} \sum_i v^{-2}(x_i) \rightarrow 0
\end{equation}
almost surely, with $N \rightarrow \infty$. 
\end{lem}
Note that since $2m>d$ by construction of the kernels, $a(N) = 1/N$ satisfies the above decay assumption. 
\begin{proof}
We have 
\begin{equation}
    \Exp{v^{-2}(x)} = \int v^{-2}(x) \cdot v^2(x) \Ind{B}(x) dx = \lambda(B) < \infty. 
\end{equation}
Thus the Marcinkiewicz-Kolmogorov law implies that 
\begin{equation}
\label{eq:marcinkiewicz_application_proof}
\frac{1}{N} \sum_i v^{-2}(x_i) \rightarrow \lambda(B)   
\end{equation}
almost surely. 
Next, recall that by Proposition \ref{prop:k_rescaling _property}, we have  $g^a(0) = a^{-\frac{d}{2m}}g^1(0)$. 
Our decay assumption on $a(N)$ implies then that 
$g^{a(N)}(0)/N \rightarrow 0$, which together with \eqref{eq:marcinkiewicz_application_proof} completes the proof.     
\end{proof}

Next, for $i \neq j$, set $Y_{ij} = v^{-1}(x_i) v^{-1}(x_j) k(x_i,x_j)$. 
\begin{lem} 
\label{lem:Y_ij_expectation_convergence}
For every $a>0$ we have $\Abs{\Exp{Y_{ij}}}\leq 1$, and moreover 
 $\Exp{Y_{ij}} \rightarrow 1 $ when $a \rightarrow 0$. 
\end{lem}
\begin{proof}
Recall that $v^2$ is a density, i.e. $\int v^2(x)dx = 1$, and that the operator $K^a$ was defined in \eqref{eq:K_a_convo_def}.
We have 
\begin{flalign}
\Abs{\Exp{Y_{ij}} - 1} &= \Abs{\int v^{-1}(x) v^{-1}(y) k^a(x,y) v^2(x) v^2(y) dx dy -  
\int v^2(x) dx }\\
&= \Abs{\int k^a(x,y) v(x) v(y) dx dy -  
\int v^2(x) dx} \\ 
&= \Abs{\int dx \cdot v(x) \SqBrack{(K_a v)(x) - v(x) } } \\ 
&\leq \norm{v}_{L_2} \norm{(K_a v) - v}_{L_2} \\ 
&=  \norm{(K_a v) - v}_{L_2}. 
\end{flalign}
The second statement now follows from Lemma \ref{lem:convolution_approximation_to_identity}.
For the first statement, write 
\begin{equation}
    \Abs{\Exp{Y_{ij}}} = 
    \inner{K_a v}{v}_{L_2} \leq \norm{K_a v}_{L_2} \norm{v}_{L_2} \leq 1,
\end{equation}
where we have used the Cauchy-Schwartz inequality and the first part of Lemma \ref{lem:convolution_approximation_to_identity}.
\end{proof}

It thus remains to establish that $\frac{1}{N^2}\sum_{i\neq j} (Y_{ij} - EY_{ij})$ converges to $0$ in probability. 
We will show this by bounding the second moment, 
\begin{equation}
\label{eq:second_moment_expression_naive}
\frac{1}{N^4}\Exp{\Brack{\sum_{i\neq j} (Y_{ij} - EY_{ij})}^2} = 
\frac{1}{N^4}
\sum_{i\neq j, i' \neq j'}\Brack{
\Exp{Y_{ij}Y_{i'j'}} - \Exp{Y_{ij}}\Exp{Y_{i'j'}}}
\end{equation}
and then using the Chebyshev inequality. 
Observe that there are three types of terms of the form $\Brack{
\Exp{Y_{ij}Y_{i'j'}} - \Exp{Y_{ij}}\Exp{Y_{i'j'}}}$ on the right hand side of 
\eqref{eq:second_moment_expression_naive}.
The first type is when $\Set{i,j} = \Set{i',j'}$ as sets. Second is when $\Abs{\Set{i,j} \cap \Set{i',j'}} = 1 $, and the third is when 
$\Set{i,j}$ and $\Set{i',j'}$ are disjoint. In the following three propositions, we bound each type of terms separately.

\begin{prop}
\label{prop:Y_ij_moment_bound}
There is a function of $m$, $c(m)>0$, such that 
\begin{equation}
    \Exp{Y_{ij}^2} \leq  \lambda(B) a^{-\frac{d}{2m}} c(m)  < \infty. 
\end{equation}
\end{prop}
\begin{proof}
Write 
\begin{flalign}
    \Exp{Y_{ij}^2} &= 
    \int v^{-2}(x)v^{-2}(y) (k^a(x,y))^{2} v^2(x) v^2(y) dx dy \\ 
    &=\int \Ind{B}(x) \Ind{B}(y) (k^a(x,y))^2  dx dy \\ 
    &\leq \int \Ind{B}(x) \norm{k_x^a}_{L_2}^2 dx \\
    &= \lambda(B) a^{-\frac{d}{2m}} c(m),
\end{flalign}
where we have used Lemma \ref{lem:k_x_a_l2_y_norm_expression} to compute $\norm{k_x^a}_{L_2}^2$. 
\end{proof}

For the $\Abs{\Set{i,j} \cap \Set{i',j'}} = 1 $ case we have 
\begin{prop} 
\label{prop:three_disjoint_indices}
Let $i,j,t$ be three distinct indices. 
Then 
\begin{equation}
    \Abs{\Exp{ Y_{ij} Y_{jt}}} \leq 1. 
\end{equation}
\end{prop}
\begin{proof}
\begin{flalign}
\Exp{ Y_{ij} Y_{jt}} &= 
\int v^{-1}(x_i) v^{-1}(x_j) k^a(x_i,x_j) 
v^{-1}(x_j) v^{-1}(x_t) k^a(x_j,x_t) 
v^2(x_i)v^2(x_j)v^2(x_t) dx_i dx_j dx_t \\ 
&= \int v(x_i) k^a(x_i,x_j) 
v(x_t) k^a(x_t,x_j) 
\Ind{B}(x_j) dx_i dx_j dx_t \\ 
&= \int (K_av)^2(x_j) \Ind{B}(x_j)  dx_j \\ 
&\leq \int (K_av)^2(x_j)  dx_j \\
&= \norm{K_a v}_{L_2}^2 \\ 
&\leq 1, 
\end{flalign}
where we have used Lemma \ref{lem:convolution_approximation_to_identity} on the last line.
\end{proof}

Finally, for the disjoint case, 
\begin{prop}
\label{prop:four_disjoint_indices}
Let $i,j,i',j'$ be four distinct indices. 
Then 
\begin{equation}
    \Exp{(Y_{ij} - EY_{ij})(Y_{i'j'} - EY_{i'j'})} = 0.
\end{equation}
\end{prop}
\begin{proof}
When $i,j,i',j'$ are distinct, $Y_{ij}$ and $Y_{i'j'}$ are independent. 
\end{proof}

We now collect these results to bound \eqref{eq:second_moment_expression_naive}. 
\begin{lem}
\label{lem:Y_ij_sum_second_moment_bound}
There is a function $c'(m,B)>0$ of $m,B$, such that, 
choosing $a(N) = 1/N$, we have 
\begin{equation}
\label{eq:second_moment_bound_lemma}
    \frac{1}{N^4}\Exp{\Brack{\sum_{i\neq j} (Y_{ij} - EY_{ij})}^2} \leq \frac{c'(m,B)}{N}
\end{equation}
for every $N>0$. 
\end{lem}
\begin{proof}
Observe first that all the expressions of the form $\Exp{Y_{ij}}\Exp{Y_{i'j'}}$ are bounded by 1, by Lemma \ref{lem:Y_ij_expectation_convergence}.  
Next, note that there are $O(N^2)$ terms of the first type. Since $2m>d$, 
we have $a^{-\frac{d}{2m}} = a^{-\frac{d}{2m}}(N) \leq N$, and thus, the overall contribution 
of such terms to the sum in \eqref{eq:second_moment_bound_lemma} is 
$O(N^{-4} \cdot N \cdot N^2) = O(1/N)$. 
Similarly, there are $O(N^3)$ terms of the second type, each bounded by constant, and thus the overall contribution is $O(N^{-4} \cdot N^3 ) = O(1/N)$. And finally, the contribution of the terms of the third type is $0$. 
\end{proof}

We now prove the main consistency Theorem. 
\begin{proof}[Of Theorem \ref{thm:consistency}]
First, observe that by definition  $\norm{u-v}_{L_2} \leq \norm{u-v}_a$ for any $u,v \in \mcH^1$. Next, by Lemma \ref{lem:strong_convexity_of_L} and using \eqref{eq:L_gradient_explicit_form}, we have 
\begin{equation}
\norm{u_N-v}_{L_2} \leq \norm{u_N-v}_a  \leq    \frac{1}{N^2}\sum_{i,j} v^{-1}(x_i) v^{-1}(x_j) k(x_i,x_j) 
- 2 + \norm{v}^2_a.
\end{equation}
Clearly, by definition \eqref{eq:norm_definition_formal}, we have $\norm{v}^2_a - 1 \rightarrow 0$ as $a \rightarrow 0$.
Write 
\begin{equation}
   \frac{1}{N^2}\sum_{i,j} v^{-1}(x_i) v^{-1}(x_j) k(x_i,x_j) 
- 1 = 
   \frac{1}{N^2}\sum_{i} Y_{i} 
    + \frac{1}{N^2}\sum_{i\neq j} (Y_{ij} - \Exp{Y_{ij}}) 
    + \SqBrack{  \frac{N^2 - N}{N^2}\cdot  \Exp{Y_{12}} - 1 }.
\end{equation}
The last term on the right hand side converges to $0$ deterministically with $a$ and $N$, by Lemma \ref{lem:Y_ij_expectation_convergence}. 
The first terms converges strongly, and hence in probability, to $0$, by Lemma \ref{lem:Y_ii_convergence_bound}.  And finally, the middle term converges in probability to $0$ by Lemma \ref{lem:Y_ij_sum_second_moment_bound} and by Chebyshev inequality. 
\end{proof}

\section{Comparison of SOSREP for Different Kernels}
\label{Supp:diverse_Kernels}
In this section, we evaluate the SOSREP density estimator with 
various kernels on the ADBench task. Specifically, we consider the SDO kernel given by \eqref{eq:kernel_fourier_inv}, the $L_2$ Laplacian kernel, and the Gaussain kernel. 
The Laplacian (aka Exponential) kernel is given by 
\begin{equation}
k(x,y) = ({1/\sigma}^d) \cdot e^{-\|x-y\|/\sigma }
\end{equation}
for $\sigma >0$, where $\|x-y\|$ is the Euclidean norm on $\RR^d$.
The Gaussian kernel is given by 
\begin{equation}
k(x,y) = ({1/\sigma}^d) \cdot e^{-\|x-y\|^2/(2\sigma^2) }.
\end{equation}
Generally, in this task, we observe a similar behavior across all kernels, although the Gaussian kernel underperforms on a few datasets.

Figure \ref{fig:supp:diverse_kernels} presents the performance of each kernel on the ADBench benchmark (see section \ref{sec:experiments_AD}). Specifically, we plot the test set AUC values on the Y-axis against the different datasets on the X-axis. Note that due to convexity, the results are nearly identical for different initializations. 
Figures \ref{fig:all_kernels_aucs} and \ref{fig:all_kernels_ranks} display the ADbench results in a manner consistent with Section \ref{sec:experiments_AD}. It is noteworthy that the Laplacian kernel also achieves impressive results. Furthermore, when considering ranks and thus mitigating the impact of extreme AUC values, the Gaussian kernel demonstrates significantly improved performance.

\begin{figure}[]
    \centering
    
    \begin{subfigure}
        \centering
        \includegraphics[scale=0.25]     {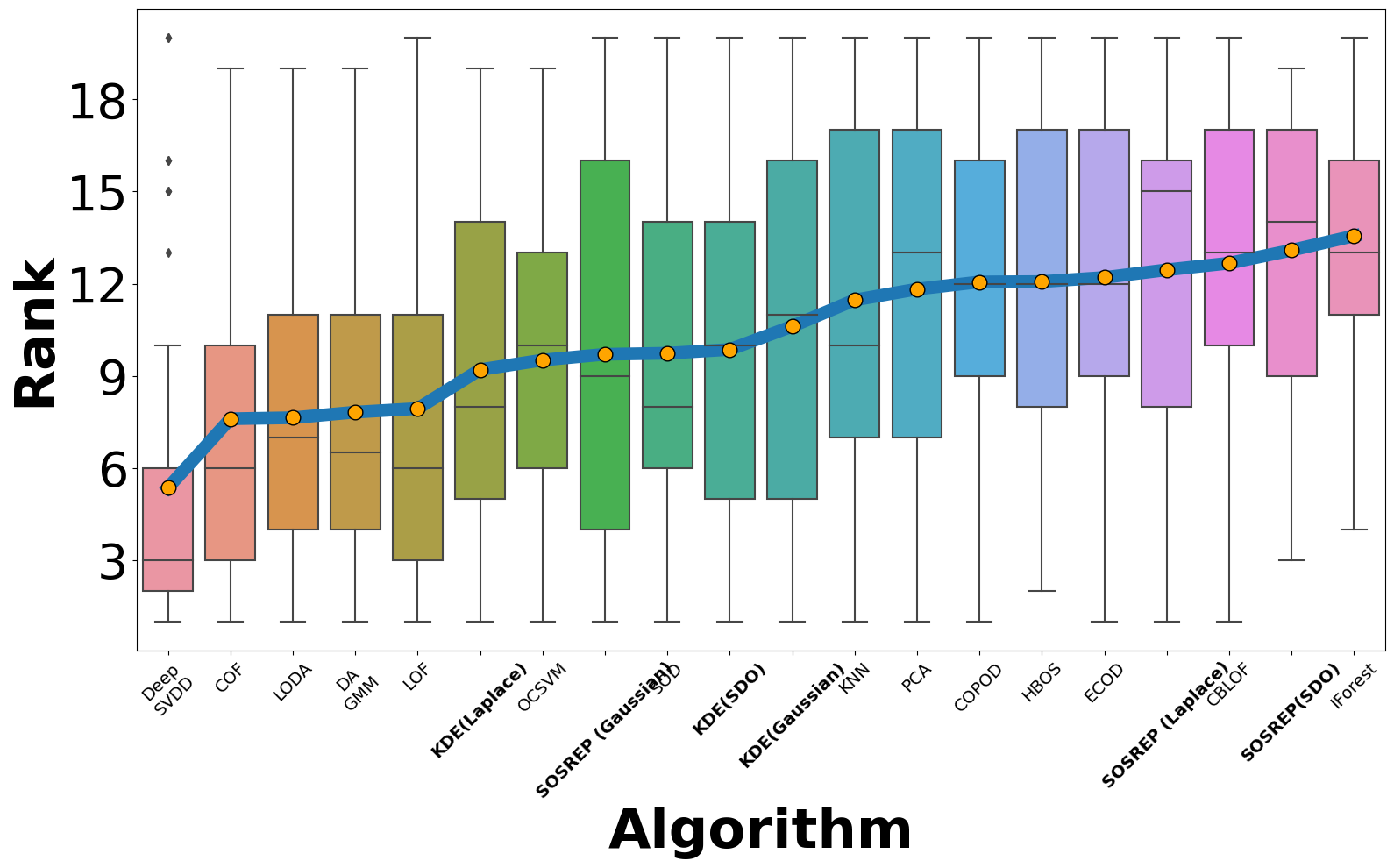}
        \caption{Anomaly Detection Results on ADBench. Mean Relative Ranking Per Dataset, Higher is Better.}
        \label{fig:all_kernels_ranks}
    \end{subfigure}
    \hfill 
    
    \begin{subfigure}
        \centering
        \includegraphics[scale=0.28]        {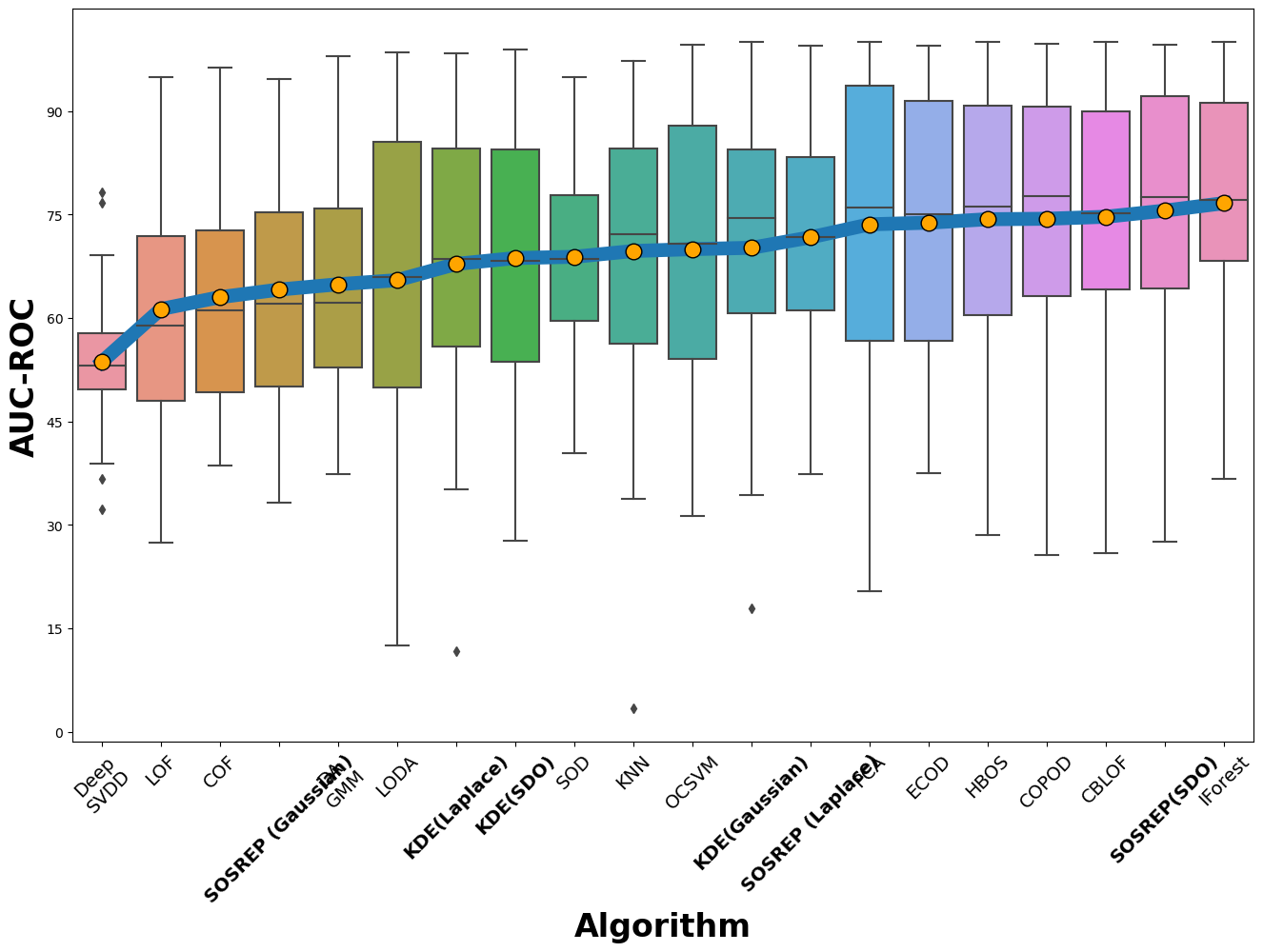}
        \caption{Anomaly Detection Results on ADBench. Mean AUC Per Dataset, Higher is Better.}
        \label{fig:all_kernels_aucs}
    \end{subfigure}
    \hfill 
    
    \begin{subfigure}
        \centering
        \includegraphics[scale=0.4]{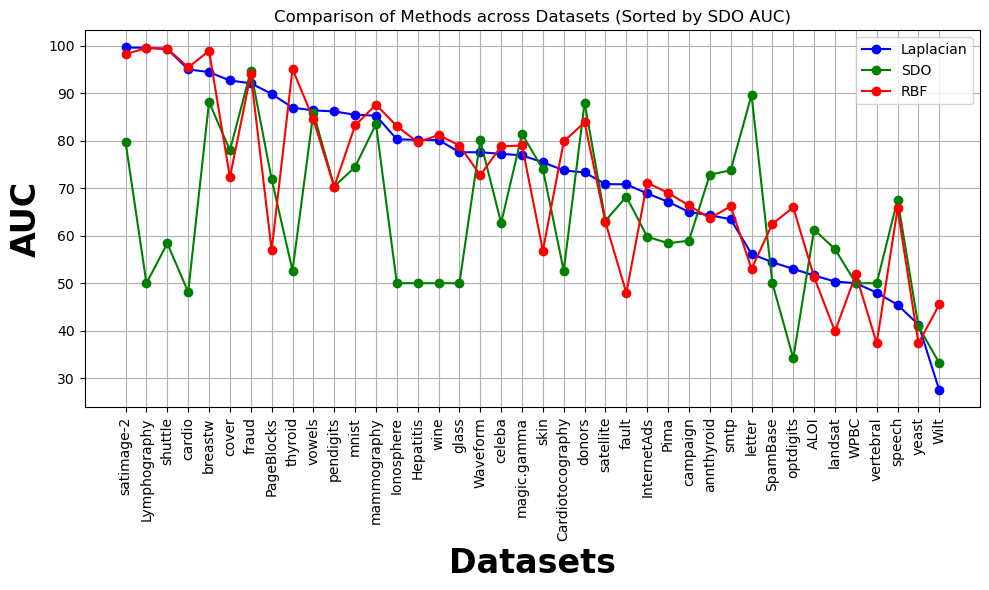}
        \caption{Comparing SOSREP results on ADBench for different kernels.}
        \label{fig:supp:diverse_kernels}
    \end{subfigure}
    
    \caption{Evaluating SOSREP with various kernels on the ADbench benchmark.}
    \label{fig:combined_figure}
\end{figure}
 
\newpage

\section{Full Table of Contents for the Supplementary materials }
\begin{enumerate}
    \item \textbf{Basic Minimizer Properties}
    \begin{itemize}
        \item Discussion on the minimizer of the SOSREP objective and its properties.
        \item Proofs for properties of the minimizer.
        \item Lemma on minimization properties.
        \item Detailed proofs for the properties of the minimizer.
    \end{itemize}
    
    \item \textbf{Derivation of the Gradients, Proof of Lemma}
    \begin{itemize}
        \item Derivation of the standard and natural gradients of the objective.
        \item Proof of the lemma regarding gradients.
    \end{itemize}
    
    \item \textbf{SDO Kernel Details}
    \begin{itemize}
        \item Full proof of Theorem on kernel form.
        \item Additional details on sampling approximation of SDO.
        \item SDO Kernel Derivation.
        \item Sampling Approximation details.
    \end{itemize}
    
    \item \textbf{A Few Basic Properties of the Kernel}
    \begin{itemize}
        \item Propositions on the real-valued nature of the kernel and its properties.
        \item Proof of properties related to the kernel.
    \end{itemize}
    
    \item \textbf{Difference between SOSREP and KDE Models}
    \begin{itemize}
        \item Analysis of the differences between SOSREP and KDE estimators.
        \item Empirical evaluation of the difference for real data.
    \end{itemize}
    
    \item \textbf{KDE vs SOSREP Comparison Proofs}
    \begin{itemize}
        \item Development of proofs for comparing KDE and SOSREP.
        \item Solution of SOSREP for a 2-Block Model.
        \item Proof of Proposition.
    \end{itemize}
    
    \item \textbf{Invariance of $\mcC$ under Natural Gradient}
    \begin{itemize}
        \item Discussion on the invariance of the non-negative cone under natural gradient steps.
        \item Proposition on invariance and its proof.
    \end{itemize}
    
    \item \textbf{Fisher Divergence for Hyper-Parameters Selection}
    \begin{itemize}
        \item The rationale behind using Fisher Divergence for hyperparameter selection.
        \item Score Matching and Fisher Divergence details.
        \item The Hessian Estimation for Small $a$'s.
        \item Approximating the Hessian Trace.
    \end{itemize}
    
    \item \textbf{Experiments}
    \begin{itemize}
        \item Detailed analysis of AUC-ROC performance.
        \item Discussion on Duplicate Anomalies and its impact.
        \item Hyperparameter Tuning with On-the-fly Fisher-Divergence Minimization.
    \end{itemize}
    
    \item \textbf{Algorithm Overview}
    \begin{itemize}
        \item Detailed overview and breakdown of the main algorithm used in the paper.
        \item Specific algorithms for sampling the multidimensional SDO Kernel, calculating optimal alphas, density over coordinates, Fisher Divergence with Hessian Trace Approximation.
    \end{itemize}
    
    \item \textbf{Figure \ref{fig:pos_neg_frac} - List of Datasets}
    \begin{itemize}
        \item Listing of datasets used for comparing fractions of negative values for alpha optimization vs natural gradients.
    \end{itemize}
    
    \item \textbf{Consistency Theorem}
    \begin{itemize}
        \item Statement and proof of the consistency result for SOSREP estimators.
        \item Detailed proof and related lemmas.
    \end{itemize}
    
    \item \textbf{Comparison of SOSREP for Different Kernels}
    \begin{itemize}
        \item Evaluation of SOSREP with various kernels on the ADBench task.
        \item Performance comparison and discussion.
    \end{itemize}
\end{enumerate}

\end{document}